\documentclass{article} 

\usepackage{iclr2027_conference,times}

\usepackage{amsmath,amsfonts,bm}

\def\eqref#1{equation~\ref{#1}}

\def\1{\bm{1}}

\DeclareMathAlphabet{\mathsfit}{\encodingdefault}{\sfdefault}{m}{sl}
\SetMathAlphabet{\mathsfit}{bold}{\encodingdefault}{\sfdefault}{bx}{n}

\DeclareMathOperator*{\argmax}{arg\,max}

\usepackage{kotex}

\usepackage{amsmath}
\usepackage{amssymb}
\usepackage{mathtools}
\usepackage{amsthm}
\usepackage{bm}
\usepackage{tabularx}
\usepackage{enumitem}

\usepackage{array}
\usepackage{booktabs}
\usepackage{multirow}
\usepackage{graphicx}
\usepackage{wrapfig}
\usepackage[most]{tcolorbox}
\usepackage{enumitem}
\usepackage{adjustbox}

\usepackage[table]{xcolor}
\definecolor{paleblue}{RGB}{70,120,180}

\usepackage{hyperref}
\usepackage{url}
\hypersetup{
    colorlinks=true,
    citecolor=paleblue,
    linkcolor=paleblue,
    urlcolor=paleblue
}

\newtheorem{theorem}{Theorem}
\newtheorem{proposition}[theorem]{Proposition}

\newcommand{\lladainst}{\mbox{LLaDA-8B-Instruct}}
\newcommand{\lladaonefive}{\mbox{LLaDA-1.5}}

\newcommand{\egi}{\text{EGI}}
\newcommand{\jg}{\text{JustGRPO}}
\newcommand{\spg}{\text{SPG}}
\newcommand{\dtwo}{\text{d2}}

\usepackage[capitalize]{cleveref}

\crefname{section}{Sec.}{Secs.}
\Crefname{section}{Sec.}{Secs.}

\crefname{figure}{Fig.}{Figs.}
\Crefname{figure}{Fig.}{Figs.}

\crefname{table}{Tab.}{Tabs.}
\Crefname{table}{Tab.}{Tabs.}

\title{Low-Confidence Remasking Traps Flexibility:\\
Realizing Arbitrary-Order Potential for\\
Diverse Rollouts in Diffusion LLMs}

\author{
Moongyu Jeon\textsuperscript{1 *} \quad
Dongjae Jeon\textsuperscript{1,2 *} \quad
Bumjun Kim\textsuperscript{1} \quad
Mingyu Kim\textsuperscript{3 \textdagger} \quad
Albert No\textsuperscript{1 \textdagger}
\\[0.8em]
\multicolumn{1}{c}{
\normalfont
\textsuperscript{1}Yonsei University
\quad\;
\textsuperscript{2}KRAFTON AI
\quad\;
\textsuperscript{3}Kookmin University
}
}

\iclrfinalcopy

\begin{document}

\maketitle
\lhead{Preprint}

\makeatletter

\renewcommand{\footnoterule}{%
  \kern-3pt
  \hrule width \textwidth
  \kern 2.6pt
}

\begingroup
\renewcommand{\thefootnote}{}
\renewcommand{\@makefntext}[1]{\noindent #1}

\footnotetext{%
\textsuperscript{*}Equal contribution.\\
\textsuperscript{\textdagger}Correspondence to Mingyu Kim
(mgyukim@kookmin.ac.kr) and Albert No
(albertno@yonsei.ac.kr).
}

\endgroup
\makeatother

\vspace{-1.5em}
\begin{abstract}
Masked diffusion language models support arbitrary-order generation, suggesting a natural way to produce diverse outputs. 
However, recent work argues that this flexibility reduces diversity by delaying high-uncertainty tokens that can lead to different generation paths. 
We trace this diversity loss not to arbitrary-order generation itself, but largely to \emph{low-confidence remasking} (LCR), a widely used decoding rule. 
At each step, LCR samples a token at every masked position but commits only the sampled token with the highest probability, filtering out the rest. 
We show that this mechanism can exponentially suppress lower-probability tokens as more positions compete, and observe the same suppression in LLaDA. 
In contrast, \emph{top-probability position selection} (TPP), which has often been conflated with LCR under the shared label \emph{confidence-based decoding}, avoids this diversity loss. 
TPP first selects the position whose most likely token has the highest probability, then samples directly from that position's distribution. 
Replacing LCR with TPP restores diversity and yields Pass@$k$ comparable to left-to-right decoding, suggesting that the reported diversity loss stems largely from LCR's filtering rather than from generating high-confidence positions first.
To further exploit order flexibility, we introduce \textbf{Entropy-Guided Initialization} (EGI), which samples the first token at the highest-entropy position and then follows TPP. 
This simple modification further improves rollout diversity and solution coverage beyond left-to-right decoding, with gains extending to downstream policy optimization, highlighting the potential of arbitrary-order generation for diverse rollouts.
\end{abstract}

\vspace{-1em}
\section{Introduction}
\label{sec:introduction}
Masked diffusion language models (MDMs) predict token distributions at all masked positions simultaneously~\citep{austin2021d3pm,lou2024sedd,sahoo2024mdlm,shi2024simplified,ou2025your}.
This enables both parallel decoding and arbitrary-order generation~\citep{wu2026fast,kim2025train}, unlike autoregressive models that generate tokens in a fixed left-to-right order~\citep{radford2018improving,radford2019language}.
In particular, this order flexibility appears naturally suited to diverse rollout sampling, where multiple responses to the same prompt explore different solution paths~\citep{gong2026diffucoder,ni2026flexibility}.
Such diversity can improve \emph{solution coverage}, measured by Pass@$k$: the probability that at least one of $k$ sampled responses is correct~\citep{chen2021humaneval}.

However, \citet{ni2026flexibility} report lower rollout diversity and Pass@$k$ under arbitrary-order decoding than under left-to-right decoding of the same MDM.
Their main comparison uses a widely adopted arbitrary-order decoding rule, \emph{low-confidence remasking} (LCR)~\citep{chang2022maskgit,nie2025large}.
They attribute this \emph{flexibility trap} to generating more certain tokens first and delaying high-uncertainty positions where token choices can lead to different reasoning paths.
As context accumulates, their uncertainty decreases, leaving fewer paths to explore.
Left-to-right decoding instead resolves uncertain positions as they arise, allowing sampling to explore different paths.

In this work, we challenge this interpretation of the flexibility trap: \emph{the reported diversity loss largely stems from how LCR filters sampled tokens, not arbitrary-order generation itself}.
For one-token decoding, LCR samples a token at every masked position but commits only the sampled token with the highest model probability.
An alternative decoding rule, \emph{top-probability position selection} (TPP)~\citep{kim2025train}, first chooses the position where the most likely token has the highest probability.
It then samples and commits a token at that position without further filtering.
Yet this distinction is often obscured in prior work, which refers to either rule as \emph{confidence-based decoding} and sometimes pairs TPP-style descriptions or analyses with LCR experiments (see App.~\ref{app:related-conflation}).

Two toy models reveal how LCR's token filtering can \emph{exponentially suppress} lower-probability tokens and outputs.
A lower-probability sampled token survives only if \emph{all} competing sampled tokens have equal or lower model probability.
In an independent-token toy model, this event becomes exponentially unlikely as the number of remaining masked positions grows.
TPP, in contrast, generates lower-probability tokens at the rate specified by the chosen temperature.
A second toy model with non-overlapping solution sequences shows analogous suppression of lower-probability outputs at the sequence level.
These examples demonstrate that LCR can severely restrict exploration even when temperature makes lower-probability tokens common among the initial samples.

We observe the same suppression in LLaDA~\citep{nie2025large}: LCR predominantly commits the most likely token at the selected position even as temperature increases.
Simply replacing LCR with TPP restores diversity and yields Pass@$k$ comparable to left-to-right decoding across benchmarks on LLaDA and LLaDA-1.5~\citep{zhu2026llada}.
Crucially, TPP still prioritizes high-confidence positions and can therefore postpone uncertain ones.
These results challenge the claim that delaying uncertain positions is the main cause of diversity loss and instead identify LCR's filtering of lower-probability sampled tokens across positions as a central limitation.
More broadly, this substantial behavioral gap makes their conflation potentially misleading, underscoring the need to distinguish LCR from TPP.

Finally, we show that changing position selection at just one decoding step can further improve rollout sampling beyond left-to-right decoding.
We introduce \emph{Entropy-Guided Initialization} (EGI), a training-free modification that changes only the first decoding step of TPP.
EGI first samples a token at the highest-entropy masked position, then follows TPP for the remaining steps.
This modification improves rollout diversity and solution coverage over left-to-right decoding, including higher Pass@$k$.
These benefits extend to downstream group-relative policy optimization, where replacing left-to-right rollouts with EGI rollouts improves performance across all three methods we evaluate~\citep{wang2026spg,wang2026d2,ni2026flexibility}.
Together, these results show that diverse rollout sampling need not be confined to left-to-right decoding and highlight the untapped potential of arbitrary-order generation as a design space for both inference and policy optimization.

Our contributions are summarized as follows:
\vspace{-0.5em}
\begin{itemize}[leftmargin=*]
    \item We distinguish two decoding rules often labeled \emph{confidence-based decoding}: low-confidence remasking (LCR) and top-probability position selection (TPP). They coincide under greedy decoding at zero temperature but define fundamentally different samplers at positive temperature, making this distinction essential for rollout sampling and their conflation potentially misleading.

    \item Using two toy models, we show that LCR's token filtering can \textit{exponentially suppress} lower-probability choices, and observe the same suppression in LLaDA. Replacing LCR with TPP largely closes the Pass@$k$ gap to left-to-right decoding, pointing to LCR's token filtering, rather than confidence-prioritized ordering itself, as a central source of the \emph{flexibility trap}.

    \item We introduce \emph{Entropy-Guided Initialization} (EGI), which modifies only the first step of TPP to use order flexibility for exploration. EGI improves rollout diversity and solution coverage beyond left-to-right decoding, with gains extending to downstream group-relative policy optimization, highlighting the potential of arbitrary-order generation for stronger rollout sampling.
\end{itemize}

\section{Preliminaries}
\label{sec:preliminaries}

\subsection{Masked Diffusion Language Models and Decoding Rules}
\label{sec:masked-diffusion-preliminaries}

Discrete diffusion models learn to generate discrete data by reversing a stochastic corruption process~\citep{hoogeboom2021argmax,austin2021d3pm,lou2024sedd}.
Masked diffusion language models (MDMs) use a forward process that progressively replaces tokens with a special mask token~\citep{sahoo2024mdlm,shi2024simplified,ou2025your}.
The model learns to reconstruct the original tokens from partially masked sequences and generates text by iteratively filling masked positions.
LLaDA~\citep{nie2025large} scales masked diffusion to an 8B-parameter large language model.

Given a partially masked sequence $\mathbf{x}$, let $M$ denote the set of masked positions and $p_i(v)$ the probability the model assigns to token $v\in\mathcal{V}$ at position $i\in M$.
An MDM predicts these distributions simultaneously, allowing tokens to be committed in parallel or in arbitrary orders~\citep{wu2026fast,kim2025train}.
To isolate decoding-rule effects, we focus on one-token-per-step decoding.

\paragraph{Decoding Rules.}
An MDM can choose a masked position to generate next and then sample a token at that position.
It can follow a left-to-right autoregressive (AR) order, always choosing the leftmost masked position, or select the next position adaptively based on its current predictions.\footnote{Block-autoregressive decoding generates contiguous blocks from left to right~\citep{arriola2025block,nie2025large}.
The position-selection criteria discussed here apply within each block.}
Common position-selection criteria favor the largest top-token probability, the largest margin between the two highest token probabilities~\citep{kim2025train}, or the lowest token entropy~\citep{ye2025dream}.
We call the first rule \emph{top-probability position selection} (TPP), which samples directly from a position's distribution after selecting it.
With consistent model conditionals, this sampling preserves the common joint distribution regardless of position order~\citep{kim2025train}.

Another widely used decoding rule is \emph{low-confidence remasking} (LCR)~\citep{chang2022maskgit,nie2025large}.
In our one-token-per-step setting, LCR first samples a token proposal $\tilde{x}_i$ at every masked position $i\in M$.
It then commits only the proposal with the highest model probability $p_i(\tilde{x}_i)$, leaving all other positions masked.\footnote{Here, ``remasking'' discards newly sampled tokens rather than masking previously committed tokens, as in methods such as ReMDM~\citep{wang2025remasking}.}
Thus, TPP selects a position before token sampling, whereas LCR selects among tokens that have already been sampled.
Despite this distinction, the term \emph{confidence-based decoding} has been used for both TPP and LCR across prior work.
We formalize these rules and examine how they differ under rollout sampling in Sec.~\ref{sec:tempered-decoding-rules}.

\subsection{Diverse Rollout Sampling}
\label{sec:rollout-sampling}

\textit{Rollout sampling} generates multiple responses to the same prompt and is one of the simplest forms of inference-time scaling.
Its effectiveness depends on whether the samples explore diverse, viable solution paths rather than repeatedly follow the same trajectory.
Such diversity can increase \emph{solution coverage}, the chance of finding at least one correct solution across multiple attempts.
This coverage is commonly measured by \textbf{Pass@$\bm{k}$}, the probability that at least one of $k$ sampled solutions is correct~\citep{chen2021humaneval,brown2024large}.

Rollout diversity is also important for Group-Relative Policy Optimization (GRPO), which samples multiple rollouts per prompt and learns from their relative rewards~\citep{shao2024deepseekmath}.
For binary rewards, groups containing both correct and incorrect rollouts provide relative learning signal, whereas identical-reward groups provide little.
By exploring different solution paths, diverse rollouts can increase the chance of within-group reward variation.
Pass@$k$ remains a common measure of such rollout coverage, while another useful metric is \textbf{Potential@$\bm{k}$}, which measures how often problems missed by a single rollout are recovered through additional sampling~\citep{yao2025diversity}.

Token temperature is one of the simplest and most common controls for rollout diversity.
Given a token distribution $p(v)$, temperature $T>0$ defines the tempered distribution
\[
p^{(T)}(v) \propto p(v)^{1/T},
\]
with $T=1$ recovering the original distribution.
As $T \to 0$, sampling becomes greedy and selects only the top-probability token, which we conventionally denote by $T=0$.
Increasing $T$ gives lower-probability tokens more mass, approaching uniform sampling as $T \to \infty$.

\section{Two Distinct Samplers Behind Confidence-based Decoding}
\label{sec:tempered-decoding-rules}

As noted in Sec.~\ref{sec:preliminaries}, both low-confidence remasking (LCR) and top-probability position selection (TPP)
have been referred to as confidence-based decoding or related terms, but they define different decoding rules.
This conflation has led to mismatches in prior work, where descriptions or analyses of TPP are sometimes paired with
experiments using LCR (see App.~\ref{app:related-conflation} for details).

To formalize the distinction, consider a single decoding step.
For a fixed partially decoded sequence, let $M$ denote the set of masked positions and $p_i(v)$ the model probability of token $v\in\mathcal V$ at position $i\in M$.
Let $p_i^{(T)}(v)\propto p_i(v)^{1/T}$ denote the tempered distribution at temperature $T>0$.

\textbf{Low-confidence remasking} (\textbf{LCR}) first independently samples a token proposal
$\tilde{x}_i\sim p_i^{(T)}$ at every masked position $i\in M$.
It scores each proposal by its probability $p_i(\tilde{x}_i)$ under the original model distribution
and selects the highest-scoring proposal:
$i^\star=\operatorname*{arg\,max}_{i\in M}p_i(\tilde{x}_i)$.
The selected proposal $\tilde{x}_{i^\star}$ is then committed to the sequence at position $i^\star$, while all other proposals are rejected.

\textbf{Top-probability position selection} (\textbf{TPP}) instead scores each masked position by its top-token probability, $c_i:=\max_{v\in\mathcal V}p_i(v)$.
It selects the highest-scoring position, $i^\star=\argmax_{i\in M} c_i$,
and then samples and immediately commits a token $x_{i^\star}\sim p_{i^\star}^{(T)}$ at that position.
All other positions in $M$ remain masked.
Notably, the committed token need not be the top token defining $c_{i^\star}$.

\noindent
\begin{minipage}[t]{0.495\linewidth}
\begin{tcolorbox}[
    enhanced,
    colback=black!4,
    colframe=black!4,
    boxrule=0pt,
    arc=1.5mm,
    left=1.8mm,
    right=1.8mm,
    top=2mm,
    bottom=2mm
]
\footnotesize
\textbf{[LCR] Sample all, then select an index}

\vspace{1.5mm}

\textit{1.\ Sample proposals:}
$\tilde{x}_i\sim p_i^{(T)},\;\forall i\in M$

\vspace{1mm}

\textit{2.\ Score and select:}
$i^\star=\arg\max_{i\in M}p_i(\tilde{x}_i)$

\vspace{1mm}

\textit{3.\ Commit:}
$x_{i^\star}=\tilde{x}_{i^\star}$
\end{tcolorbox}
\end{minipage}
\hfill
\begin{minipage}[t]{0.495\linewidth}
\begin{tcolorbox}[
    enhanced,
    colback=black!4,
    colframe=black!4,
    boxrule=0pt,
    arc=1.5mm,
    left=1.8mm,
    right=1.8mm,
    top=2mm,
    bottom=2mm
]
\footnotesize
\textbf{[TPP] Select a position, then sample}

\vspace{1.5mm}

\textit{1.\ Score positions:}
$c_i:=\max_{v\in\mathcal V}p_i(v),\;\forall i\in M$

\vspace{1mm}

\textit{2.\ Select a position:}
$i^\star=\arg\max_{i\in M}c_i$

\vspace{1mm}

\textit{3.\ Sample and commit:}
$x_{i^\star}\sim p_{i^\star}^{(T)}$
\end{tcolorbox}
\end{minipage}

Under greedy decoding at $T=0$, LCR and TPP coincide, committing the highest-probability token across all masked positions.
This equivalence helps explain why the two procedures can be conflated under deterministic decoding.
For rollout sampling at $T>0$, however, their difference affects the sampling distribution, not merely the decoding order.

\emph{TPP preserves the model distribution in principle.}
At $T=1$, TPP changes only which position is generated next, so if the model conditionals come from a common joint distribution, sampling order does not change the resulting joint distribution~\citep{kim2025train}.
For any $T>0$, TPP samples and commits directly from the tempered distribution at the selected position, following the token probabilities induced by the temperature without further rejection based on the sampled token.

\emph{LCR, by contrast, can exponentially suppress lower-probability tokens.}
A lower-probability proposal is committed only if \emph{every} competitor has lower model probability.
A single higher-probability proposal suffices to reject it, so its commitment probability rapidly decreases as more positions compete.
LCR therefore biases commitment toward higher-probability proposals, severely suppressing the lower-probability tokens introduced by the requested temperature~\citep{zhang2026differences}.

Sec.~\ref{sec:temperature-suppression-analysis} establishes this exponential suppression in toy models as the number of competing positions grows.
Sec.~\ref{sec:empirical-lcr} observes the corresponding pattern in large-scale MDMs (e.g., LLaDA), where commitments remain concentrated on top-probability tokens even as temperature increases.

\section{Analyzing Diversity Suppression in Low-Confidence Remasking}
\label{sec:temperature-suppression-analysis}

We use two toy models to isolate how LCR's rejection of lower-probability proposals affects rollout diversity.
The first considers i.i.d.\ tokens, where each decoding step leaves the remaining token distributions unchanged, and reveals the effect at the token level.
This setting closely matches the independent-position analysis of \citet{zhang2026differences}, but yields a sharper result: beyond entropy reduction, non-top commitments are exponentially suppressed under LCR.
The second considers fully non-overlapping sequences and extends the same mechanism to sequence-level outputs.
In both settings, we reduce the analysis to top-probability versus non-top choices and compare how LCR suppresses non-top choices while TPP follows their tempered probability.
See App.~\ref{app:related-flexibility} for detailed comparisons with prior work and App.~\ref{app:theory} for proofs and additional results.

\subsection{Toy Model I: Independent Tokens}
\label{sec:iid-kary}

\paragraph{Setup.}
We begin with the simplest i.i.d.\ setting, where every position independently follows the same strictly positive distribution $p$ over a finite vocabulary $\mathcal V$, with a unique top-probability token.
We generate a length-$L$ sequence from a fully masked state, committing one token per step.
Let $c_T:=\max_{v\in\mathcal V}p^{(T)}(v)$ denote the top-token probability under the tempered distribution.

\paragraph{Non-top tokens are exponentially suppressed under LCR.}
Consider a single decoding step with $m$ masked positions.
Under TPP, all positions have the same top-token probability, so any position can be selected and a token is sampled and committed directly from $p^{(T)}$.
A non-top token is therefore committed with probability $1-c_T$, directly reflecting the chosen temperature.
Under LCR, however, a non-top proposal is committed if and only if all $m$ positions propose non-top tokens, since a single top-token proposal is enough to reject it.
Thus, the probabilities of committing a non-top token under TPP and LCR are

\[
\boxed{
\mathbb{P}_{\mathrm{TPP}}(\text{non-top token})
=
1-c_T,
\quad
\mathbb{P}_{\mathrm{LCR}}(\text{non-top token})
=
(1-c_T)^m.
}
\]

TPP therefore commits non-top tokens at exactly the rate specified by temperature, whereas under LCR the same probability collapses exponentially as more positions compete.

\paragraph{The final non-top fraction vanishes under LCR.}
The stepwise suppression accumulates over the full sequence.
The expected fraction of non-top tokens in the final sequence satisfies
\[
\boxed{
\mathbb{E}_{\mathrm{TPP}}
\left[
\frac{N_{\text{non-top}}}{L}
\right]
=
1-c_T,
\quad
\mathbb{E}_{\mathrm{LCR}}
\left[
\frac{N_{\text{non-top}}}{L}
\right]
\leq
\frac{1-c_T}{Lc_T}
=
O\!\left(\frac{1}{L}\right),
}
\]
where $N_{\text{non-top}}$ is the number of non-top tokens in the final sequence.
Under TPP, non-top tokens appear in the final sequence at exactly the expected fraction $1-c_T$ specified by the tempered distribution.
Under LCR, however, the same temperature drives the non-top fraction to zero as $O(1/L)$, even though the tempered distribution continues to assign non-top tokens a fixed probability $1-c_T$.

\paragraph{Concrete Example.}
Tab.~\ref{tab:toy-examples} (Left) illustrates the gap between TPP and LCR for $L=128$ with 20 tokens, where the top-probability token has probability only $0.1$ and the other 19 equally share the remaining $0.9$.
As $T\to\infty$, the tempered distribution assigns $95\%$ probability to non-top tokens, which TPP preserves.
Despite this extreme flattening, the probability that LCR commits a non-top token at the first step falls to \textbf{$0.141\%$}, and the final non-top fraction is \textbf{$14.82\%$}.

\begin{table*}[t]
\centering

\caption{
\textbf{Concrete toy examples of diversity suppression under LCR.}
For $L=128$, Toy I reports the non-top token probability at the
first decoding step and the expected non-top fraction in the
final sequence.
Toy II reports the non-top sequence probability for ten
non-overlapping sequences.
Target (TPP) denotes the corresponding values under the tempered
distribution, which TPP preserves exactly, whereas LCR suppresses
non-top choices through cross-position competition.
}
\label{tab:toy-examples}

\vspace{0.5em}

\setlength{\tabcolsep}{2.1pt}
\renewcommand{\arraystretch}{1.18}

\newcommand{\toyheader}[1]{{\fontsize{6.8}{7.0}\selectfont #1}}
\newcommand{\toydata}[1]{{\fontsize{6.5}{7.0}\selectfont #1}}

\resizebox{0.95\textwidth}{!}{%
\begin{tabular}{|
c|
c
!{\color{black!25}\vrule width 0.4pt}
c|
c
!{\color{black!25}\vrule width 0.4pt}
c|
c
!{\color{black!25}\vrule width 0.4pt}
c|}
\cline{2-7}

\multicolumn{1}{c|}{}
& \multicolumn{4}{c|}{
    \toyheader{\textbf{Independent Token Generation (Toy I)}}}
& \multicolumn{2}{c|}{
    \toyheader{\textbf{Non-Overlapping Sequence Generation (Toy II)}}} \\
\hline

\toyheader{\textbf{$T$}}
& \multicolumn{2}{c|}{
    \toyheader{$\mathbb{P}(\text{non-top token})$}}
& \multicolumn{2}{c|}{
    \toyheader{$\mathbb{E}[N_{\text{non-top}}/L]$}}
& \multicolumn{2}{c|}{
    \toyheader{$\mathbb{P}(\text{non-top sequence})$}} \\
\cline{2-7}

& \toyheader{\textbf{LCR}}
& \toyheader{\textbf{Target (TPP)}}
& \toyheader{\textbf{LCR}}
& \toyheader{\textbf{Target (TPP)}}
& \toyheader{\textbf{LCR}}
& \toyheader{\textbf{Target (TPP)}} \\
\hline

\toydata{$1$}
& \toydata{$0.000139\%$}
& \toydata{$90.0\%$}
& \toydata{$7.03\%$}
& \toydata{$90.0\%$}
& \toydata{$0.0000000000394\%$}
& \toydata{$80.0\%$} \\

\toydata{$2$}
& \toydata{$0.00801\%$}
& \toydata{$92.90\%$}
& \toydata{$10.22\%$}
& \toydata{$92.90\%$}
& \toydata{$0.000000270\%$}
& \toydata{$85.71\%$} \\

\toydata{$\infty$}
& \toydata{$0.141\%$}
& \toydata{$95.0\%$}
& \toydata{$14.82\%$}
& \toydata{$95.0\%$}
& \toydata{$0.000139\%$}
& \toydata{$90.0\%$} \\

\hline
\end{tabular}%
}

\end{table*}

\subsection{Toy Model II: Non-Overlapping Sequences}
\label{sec:k-solution}

\paragraph{Setup.}
We next consider a sequence setting motivated by reasoning problems with distinct valid solution trajectories.
Let $\{\mathbf{x}_{(1)},\ldots,\mathbf{x}_{(K)}\}$ be a finite set of length-$L$ sequences with positive probabilities, one of which we generate from a fully masked state one token per step.
For simplicity, we assume that any two sequences differ at every position:
$(\mathbf{x}_{(k)})_i \neq (\mathbf{x}_{(\ell)})_i$ for all $k\neq \ell$ and all $i$.
Assume a unique top-probability sequence, and let $\pi_T$ denote its probability at temperature $T$.

\paragraph{Non-top sequence probability is exponentially suppressed.}
By construction, the first committed token identifies the sequence and fixes its continuation.
Under TPP, all positions have the same top-token probability, so any position can be selected.
Because the first committed token determines the sequence, TPP generates a non-top sequence with probability $1-\pi_T$.
Under LCR, however, a non-top sequence is generated only when all $L$ positions propose tokens from non-top sequences, since a single top-sequence proposal rejects every non-top proposal.
Thus, the probabilities of generating a non-top sequence under TPP and LCR are
\[
\boxed{
\mathbb{P}_{\mathrm{TPP}}
(\text{non-top sequence})
=
1-\pi_T,
\quad
\mathbb{P}_{\mathrm{LCR}}
(\text{non-top sequence})
=
(1-\pi_T)^L.
}
\]
TPP generates non-top sequences at exactly the probability specified by temperature, whereas under LCR that probability collapses exponentially with sequence length.

\paragraph{Concrete Example.}
Tab.~\ref{tab:toy-examples} (Right) considers ten non-overlapping length-$128$ sequences, where the top-probability sequence has probability $0.2$ and the other nine equally share $0.8$.
As $T\to\infty$, the tempered distribution assigns $90\%$ probability to non-top sequences.
TPP preserves this probability, whereas LCR suppresses it to \textbf{$0.000139\%$}.

Across both toy models, TPP preserves the non-top probability specified by temperature, whereas LCR can suppress it exponentially through competition across token proposals.
By rejecting lower-probability proposals before commitment, LCR can therefore severely constrain the token and sequence choices that reach the final output, even at high temperature.

\section{Low-Confidence Remasking Suppresses Diversity in Practice}
\label{sec:empirical-lcr}
\begin{figure*}[t]
    \centering
    \includegraphics[width=\textwidth]{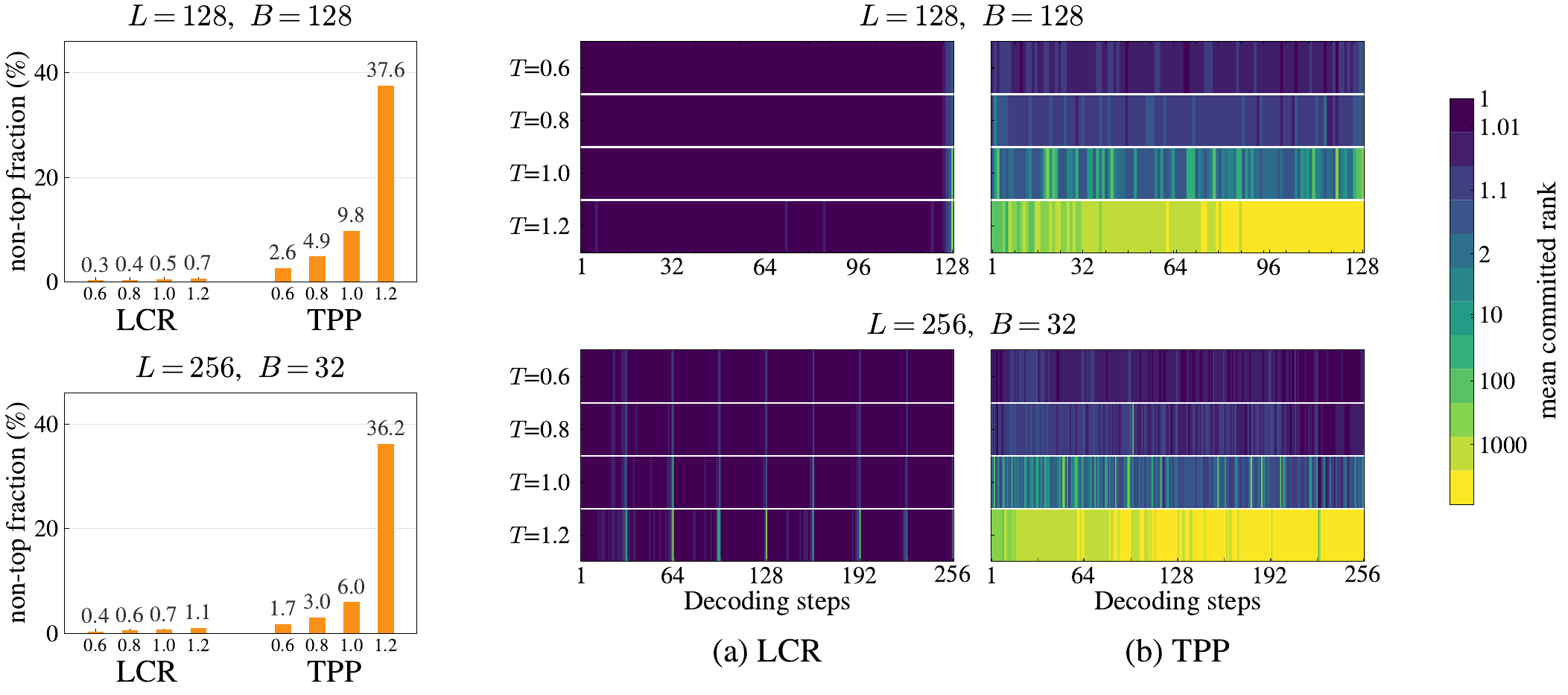}
    
\caption{
(\textbf{Left}) \textbf{Non-top commitment rate.}
Higher temperature substantially increases non-top commitments under TPP but barely under LCR.
Bars show the fraction of committed tokens that are non-top.
(\textbf{Right}) \textbf{Committed-token rank over decoding.}
LCR remains near rank one until the end of each sequence or block, whereas TPP commits higher-rank tokens throughout decoding as temperature increases.
Heatmaps show the per-step mean local rank of committed tokens.
Both panels use GSM8K with $L/B=128/128$ and $256/32$.
HumanEval results are provided in~\cref{fig:non_top_fraction_he}.
}

\vspace{0em}
    \label{fig:non_top_fraction_gsm}
\end{figure*}

Our analysis in Sec.~\ref{sec:temperature-suppression-analysis} predicts strong suppression of lower-probability proposals under LCR before commitment.
We test whether the same suppression occurs in a large-scale MDM, LLaDA-8B-Instruct~\citep{nie2025large}, using GSM8K~\citep{cobbe2021gsm8k} and HumanEval~\citep{chen2021humaneval}.
We generate $64$ rollouts per prompt with generation length/block size $128/128$ and $256/32$.
See App.~\ref{app:lcr-details} for experimental details and additional results.

\paragraph{Non-top commitment rate.}
We compare LCR and TPP across four temperatures $T\in\{0.6,0.8,1.0,1.2\}$ and measure the fraction of \emph{non-top} commitments: tokens that are not the model's highest-probability token at their position at the time of commitment.
\Cref{fig:non_top_fraction_gsm} (Left) shows that increasing temperature sharply increases this fraction under TPP, reaching $37.6\%$ and $36.2\%$ at $T=1.2$ for the $128/128$ and $256/32$ settings, whereas LCR remains at only $0.7\%$ and $1.1\%$.
\Cref{fig:non_top_fraction_gsm} (Right) further tracks the mean local rank of the committed token at its position over decoding steps.
Under LCR, this rank remains almost entirely at \textbf{one} even as temperature increases, indicating near-greedy token selection.
Ranks rise mainly near the end of each sequence or block, when only a few masked positions remain to compete.
Under TPP, by contrast, higher temperature raises the committed-token rank throughout decoding.
This pattern is consistent with~\cref{sec:iid-kary}: cross-position competition strongly suppresses non-top commitments under LCR until few competitors remain.

\begin{wrapfigure}{r}{0.4\columnwidth}
    \centering
    \vspace{-1.2em}

        \includegraphics[width=\linewidth]{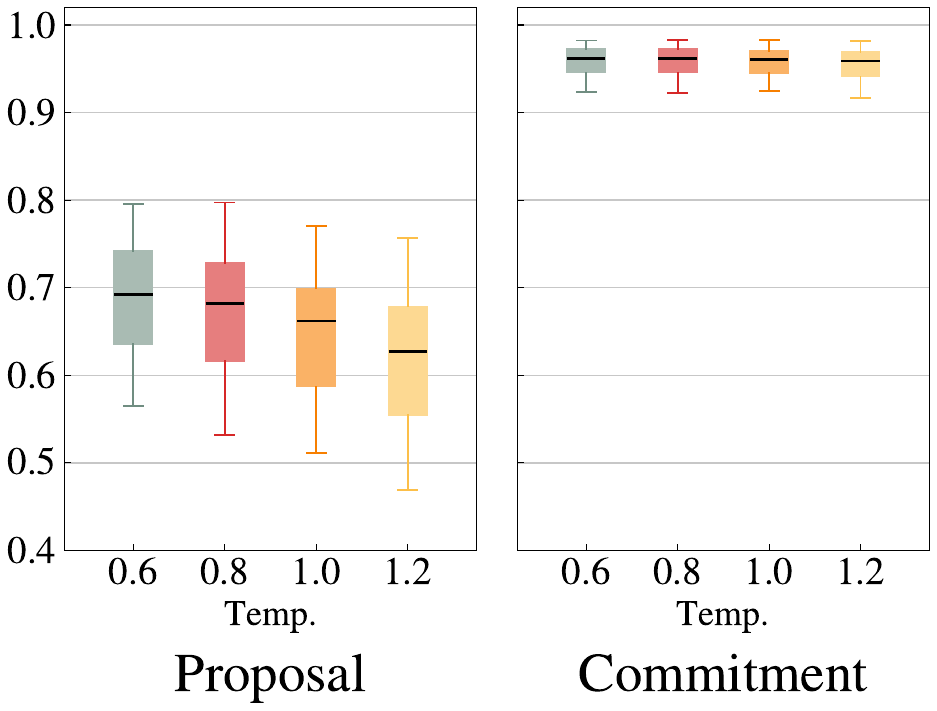}

    \vspace{-0.6em}

\caption{\hspace{-0.3em}
\textbf{LCR\hspace{-0.1em} proposal/commitment.}
Box plots show untempered model probabilities of sampled proposals (\textit{proposal}) and committed tokens (\textit{commitment}), pooled over decoding steps.
GSM8K results with $L/B=256/32$.
HumanEval results in App.~\ref{app:lcr-details}.
}

    \label{fig:lcr-proposal-commit-gsm}
    \vspace{-3em}
\end{wrapfigure}

\paragraph{Proposal vs. commitment under LCR.}
\Cref{fig:lcr-proposal-commit-gsm} compares the untempered model
probabilities of LCR's sampled proposals and committed tokens,
pooled across decoding steps.
As temperature increases, proposal probabilities shift downward,
showing that LCR samples more lower-probability tokens.
Yet committed-token probabilities remain concentrated at high values,
typically above $0.9$, across temperatures.
Temperature thus changes what LCR proposes far more than what it commits. 
Lower-probability alternatives are sampled, then largely rejected before they can enter the generated sequence.

Despite this suppression, LCR can still benefit from multiple rollouts through variation in generation order and occasional non-top commitments near the end of each block (see App.~\ref{app:lcr-order-diversity}).
However, these diagnostics show that most token-level exploration introduced by temperature is removed before commitment, substantially limiting the diversity contributed by token sampling.

\section{Realizing the Potential of Arbitrary-Order Rollouts}
\label{sec:flexibility-revisited}

Secs.~\ref{sec:temperature-suppression-analysis} and~\ref{sec:empirical-lcr} showed that LCR strongly suppresses lower-probability token commitments in both toy models and a large-scale MDM, revealing an inherent bias toward high-probability tokens.
We now revisit the reported advantage of left-to-right (AR) over arbitrary-order generation in rollout coverage~\citep{ni2026flexibility}.
We first examine whether this gap persists when LCR is replaced with TPP, and then whether even minimal use of arbitrary-order flexibility can further improve rollout diversity, solution coverage, and downstream policy optimization.
See Apps.~\ref{app:experimental-details} and~\ref{app:sec-exp-details-grpo} for details.

\subsection{Top-Probability Position Selection Closes the Rollout-Coverage Gap}
\label{sec:tpp-closes-gap}

\paragraph{Setup.}
We compare AR, LCR, and TPP using Pass@$k$ on \lladainst~\citep{nie2025large} and
\lladaonefive~\citep{zhu2026llada} across GSM8K~\citep{cobbe2021gsm8k},
MATH-500~\citep{hendrycks2021math}, HumanEval~\citep{chen2021humaneval},
and MBPP~\citep{austin2021mbpp}.
For each prompt, we generate $64$ rollouts and report Pass@$k$ for $k\leq64$.
Following~\citet{ni2026flexibility}, AR, LCR, and TPP use $T=0.6$.
Fig.~\ref{fig:passk-results} also includes the proposed method \textbf{EGI}, which we introduce in Sec.~\ref{sec:egi}.

\paragraph{Results.}
\Cref{fig:passk-results} shows a consistent pattern across benchmarks and models: LCR yields substantially weaker Pass@$k$, whereas TPP achieves Pass@$k$ comparable to AR.
This indicates that the token suppression observed under LCR is associated with reduced solution coverage across rollouts.

These results revisit the conclusion of \citet{ni2026flexibility}, who attribute the higher Pass@$k$ of AR over LCR to confidence-based decoding postponing uncertain positions.
Crucially, TPP also prioritizes positions with high top-token probabilities and can therefore postpone uncertain positions, yet achieves Pass@$k$ comparable to AR.
These results indicate that the diversity loss arises from LCR's proposal-rejection mechanism rather than from confidence-prioritized position selection itself.

\begin{figure*}[t]
    \centering
    \includegraphics[width=\textwidth]{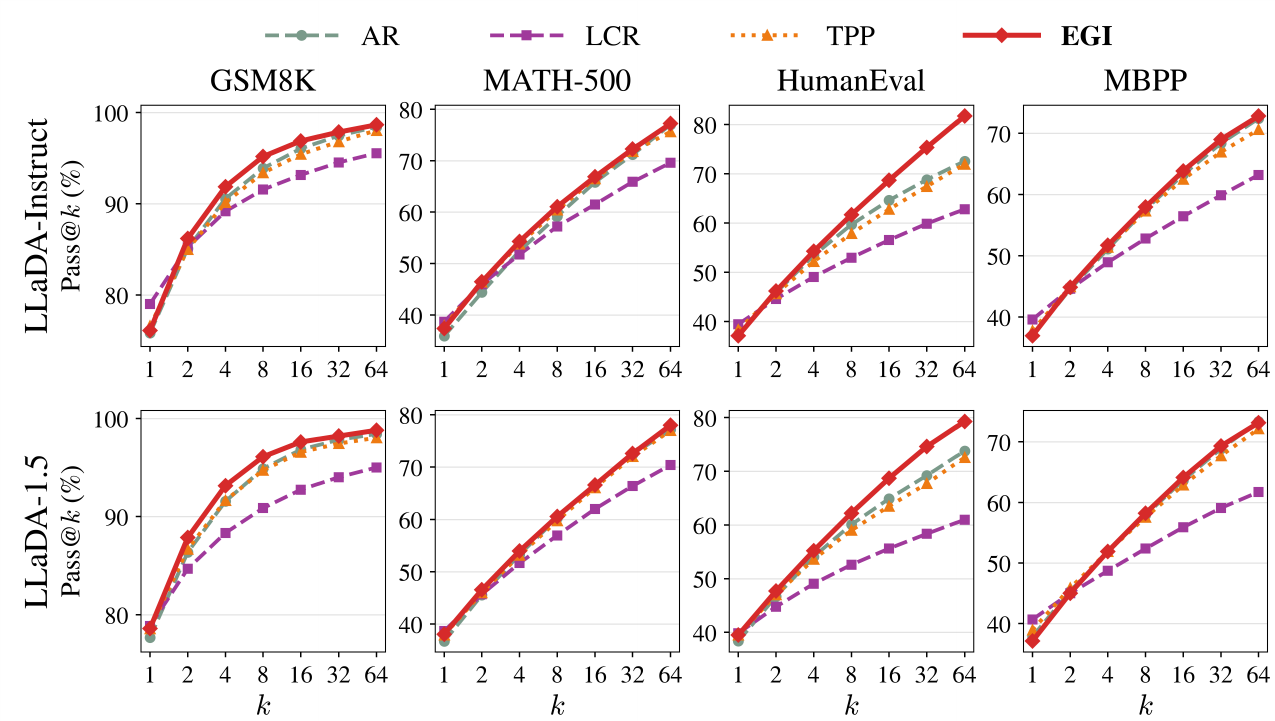}
\vspace{-1em}
\caption{\textbf{Pass@$\bm{k}$ results.}
Pass@$k$ for AR and three arbitrary-order decoding rules (LCR, TPP, EGI) with
\lladainst~and \lladaonefive.
LCR shows substantially weaker rollout coverage, TPP recovers performance comparable to AR, and EGI further improves Pass@$k$.
For LCR, TPP, and EGI, the generation length/block size is set to $256/32$.
}
\vspace{-0.8em}
    \label{fig:passk-results}
\end{figure*}

\subsection{A Single Use of Order Flexibility Surpasses Autoregressive Rollouts}
\label{sec:egi}

We next ask whether arbitrary-order flexibility can be leveraged further to improve rollout exploration.
TPP prioritizes high-confidence positions and is not explicitly designed to promote diversity.
We show that exploring this order flexibility only once, by changing the first committed position, can improve rollout diversity and solution coverage beyond AR.

We propose \textbf{Entropy-Guided Initialization} (EGI), which selects the highest-entropy position at the first decoding step,
$i^\star=\argmax_{i\in M}-\sum_{v\in\mathcal V}p_i(v)\log p_i(v)$,
samples and commits a token $x_{i^\star}\sim p_{i^\star}^{(T)}$, and follows TPP thereafter.
In the rollout experiments in this subsection, the first commitment uses $T=0.9$, while the subsequent TPP steps use $T=0.6$.
EGI thus uses order flexibility explicitly only at the first commitment, then follows regular TPP for the remaining steps.

\paragraph{Rollout diversity.}

\begin{wraptable}{r}{0.48\columnwidth}
    \centering
    \vspace{-2.2em}

    \caption{\textbf{Rollout diagnostics for EGI and AR.}
    Div-SB: Div-Self-BLEU; Pot.@$k$: Potential@$k$
    Higher values are better. HE denotes HumanEval. 
    Metric definitions in App.~\ref{app:diversity_exp_detail}.
    }
    \label{tab:rollout_diversity}

    \vspace{0.4em}
    \fontsize{8}{9.5}\selectfont
    \setlength{\tabcolsep}{3.0pt}
    \renewcommand{\arraystretch}{1.20}

    \resizebox{\linewidth}{!}{%
    \begin{tabular}{
        |>{\centering\arraybackslash}p{4.5mm}|
        c|
        c
        !{\color{black!25}\vrule width 0.4pt}
        c
        !{\color{black!25}\vrule width 0.4pt}
        c
        !{\color{black!25}\vrule width 0.4pt}
        c|
    }

        \cline{2-6}
        \multicolumn{1}{c|}{}
        & \textbf{Sampler}
        & \textbf{Div-SB}
        & \textbf{Pot.@8}
        & \textbf{Pot.@16}
        & \textbf{Pot.@64} \\
        \hline

        \multirow{2}{4.5mm}{%
            \centering
            \rotatebox[origin=c]{90}{%
                \fontsize{7}{8}\selectfont\bfseries GSM8K}}
        & AR
        & 52.3
        & 74.8
        & 83.0
        & 93.0
        \rule[-1.35ex]{0pt}{4.5ex} \\

        & EGI
        & \textbf{58.1}
        & \textbf{80.4}
        & \textbf{86.8}
        & \textbf{94.0}
        \rule[-1.35ex]{0pt}{4.5ex} \\

        \hline

        \multirow{2}{4.5mm}{%
            \centering
            \rotatebox[origin=c]{90}{%
                \fontsize{7}{8}\selectfont\bfseries HE}}
        & AR
        & 57.2
        & 34.5
        & 41.8
        & 54.5
        \rule[-1.35ex]{0pt}{4.5ex} \\

        & EGI
        & \textbf{58.8}
        & \textbf{38.5}
        & \textbf{49.1}
        & \textbf{69.7}
        \rule[-1.35ex]{0pt}{4.5ex} \\

        \hline
    \end{tabular}%
    }

    \vspace{-1em}
\end{wraptable}

\Cref{fig:passk-results} shows that EGI improves Pass@$k$ over both TPP and AR, with particularly clear gains on GSM8K and HumanEval.
We further compare EGI with AR using \textbf{Div-Self-BLEU}~\citep{zhu2018texygen,yao2025diversity}, which measures lexical diversity across rollouts, and \textbf{Potential@$k$}~\citep{yao2025diversity}, which measures how often problems missed by Pass@1 are recovered within $k$ rollouts.
\Cref{tab:rollout_diversity} confirms that EGI also exceeds AR in both metrics, indicating that even a single use of order flexibility can improve rollout diversity and solution coverage.

\paragraph{Effect of the first-step intervention.}




\begin{wraptable}{r}{0.42\columnwidth}
    \centering
    \vspace{-2.2em}

    \caption{\textbf{Decoding-order divergence.}
Mean pairwise normalized edit distance between commitment orders
across rollouts (mean $\pm$ SE).
Higher values indicate more diverse decoding orders.}
    \label{tab:tpp-egi-divergence}

    \vspace{0.4em}
    \fontsize{8}{9.5}\selectfont
    \setlength{\tabcolsep}{3.0pt}
    \renewcommand{\arraystretch}{1.20}

    \resizebox{\linewidth}{!}{%
    \begin{tabular}{|
        >{\centering\arraybackslash}p{12mm}|
        >{\centering\arraybackslash}p{22mm}
        !{\color{black!25}\vrule width 0.4pt}
        >{\centering\arraybackslash}p{22mm}|}

        \hline
        \textbf{Sampler}
        & \textbf{GSM8K}
        & \textbf{HumanEval} \\
        \hline

        TPP
        & 0.45 {\scriptsize $\pm 0.001$}
        & 0.39 {\scriptsize $\pm 0.006$}
        \rule[-1.1ex]{0pt}{4.0ex} \\

        \arrayrulecolor{black!25}
        \cline{1-3}
        \arrayrulecolor{black}

        EGI
        & \textbf{0.52} {\scriptsize $\pm 0.0006$}
        & \textbf{0.52} {\scriptsize $\pm 0.001$}
        \rule[-1.1ex]{0pt}{4.0ex} \\

        \hline
    \end{tabular}%
     }

    \vspace{-2em}
\end{wraptable}

To analyze how EGI's single-token intervention affects TPP,
we measure the mean pairwise normalized edit distance between the decoding-order permutations of different rollouts, with larger values indicating more divergent orders (see App.~\ref{app:diversity_exp_detail} for metric definitions).
\Cref{tab:tpp-egi-divergence} shows that EGI produces more diverse decoding orders than TPP despite differing only at the first step.

To test whether the gains come simply from increasing first-step temperature, we apply the same temperature increase to AR and observe much smaller gains (see App.~\ref{app:ar-firsttoken}).
These results suggest that selecting a high-entropy initial position through order flexibility contributes to EGI's gains beyond the temperature increase alone.

\subsection{Entropy-Guided Rollouts Improve Policy Optimization}
\label{sec:grpo}

We finally show that the broader rollout exploration provided by EGI translates into improved downstream performance under group-relative policy optimization (GRPO)~\citep{shao2024deepseekmath}.

\paragraph{Setup.}

\begin{wraptable}{r}{0.50\textwidth}
\vspace{-2.2em}
\centering

\caption{\textbf{RL results with AR and EGI rollouts.}
Best evaluation accuracy across checkpoints evaluated at fixed FLOP intervals ($1\times10^{19}$) for SPG, JustGRPO, and d2.
Within each method, only the training rollout sampler differs between AR and EGI.
The better result is shown in \textbf{bold}.}
\label{tab:rl_method_comparison}

\vspace{0.8em}
\fontsize{8.5}{10}\selectfont
\setlength{\tabcolsep}{2.1pt}
\renewcommand{\arraystretch}{1.18}

\resizebox{\linewidth}{!}{%
\begin{tabular}{|
    >{\centering\arraybackslash}p{4.2mm}|
    c|
    c
    !{\color{black!25}\vrule width 0.4pt}
    c
    !{\color{black!25}\vrule width 0.4pt}
    c
    !{\color{black!25}\vrule width 0.4pt}
    c|}

    \cline{2-6}
    \multicolumn{1}{c|}{}
    & \textbf{Sampler}
    & \textbf{GSM8K}
    & \textbf{MATH-500}
    & \textbf{HumanEval}
    & \textbf{MBPP} \\
    \hline

    \multirow{2}{4.2mm}{%
        \centering
        \rotatebox[origin=c]{90}{%
            \fontsize{7}{8}\selectfont\bfseries SPG}}
    & AR
    & 82.11
    & 37.8
    & 38.4
    & 42.8
    \rule[-1.05ex]{0pt}{4.15ex} \\

    & EGI
    & \textbf{84.08}
    & \textbf{39.6}
    & \textbf{40.2}
    & \textbf{43.4}
    \rule[-1.05ex]{0pt}{4.15ex} \\

    \arrayrulecolor{black!25}
    \hline
    \arrayrulecolor{black}

    \multirow{2}{4.2mm}[0.8mm]{%
        \centering
        \rotatebox[origin=c]{90}{%
            \fontsize{6.4}{7.2}\selectfont\bfseries JustGRPO}}
    & AR
    & 81.05
    & 36.8
    & 37.8
    & 44.0
    \rule[-1.05ex]{0pt}{4.15ex} \\

    & EGI
    & \textbf{83.17}
    & \textbf{38.2}
    & \textbf{39.0}
    & \textbf{44.6}
    \rule[-1.05ex]{0pt}{4.15ex} \\

    \arrayrulecolor{black!25}
    \hline
    \arrayrulecolor{black}

    \multirow{2}{4.2mm}{%
        \centering
        \rotatebox[origin=c]{90}{%
            \fontsize{7}{8}\selectfont\bfseries d2}}
    & AR
    & 81.73
    & 37.2
    & 37.2
    & 43.6
    \rule[-1.05ex]{0pt}{4.15ex} \\

    & EGI
    & \textbf{85.29}
    & \textbf{38.8}
    & \textbf{38.4}
    & \textbf{44.2}
    \rule[-1.05ex]{0pt}{4.15ex} \\

    \hline

\end{tabular}%
}

\vspace{-2em}
\end{wraptable}

We compare AR and EGI for rollout generation using three recent group-relative policy optimization methods for large MDMs: \spg~\citep{wang2026spg}, \jg~\citep{ni2026flexibility}, and \dtwo~\citep{wang2026d2}.
For each method, we vary only the rollout decoding rule, using either AR or EGI, while keeping the policy optimization objective unchanged.
We evaluate on GSM8K, MATH-500, HumanEval, and MBPP.
For the coding tasks, training uses a subset of AceCoder-87K~\citep{zeng2025acecoder}, following \citet{ni2026flexibility,gong2026diffucoder,ou2026principled}.

Following~\citet{wang2026d2}, we match total training compute across methods and evaluate checkpoints at fixed FLOP intervals, reporting the best evaluation accuracy.
All training rollouts use group size $G=8$ and a base temperature of $T=0.6$, matching the best AR configuration in~\citet{ni2026flexibility}.
For \egi{}, the first commitment uses $T=0.9$ for math and $T=0.6$ for code.
Evaluation uses deterministic LCR/TPP ($T=0$) for all methods, which are equivalent under greedy decoding, consistent with prior post-RL evaluation protocols~\citep{zhao2025d1,tang2025wd1,ni2026flexibility}.
Experiment details in App.~\ref{app:sec-exp-details-grpo}.

\paragraph{Results.}

\begin{wrapfigure}{r}{0.39\columnwidth}
    \centering
    \vspace{-1.2em}
    \includegraphics[width=\linewidth]{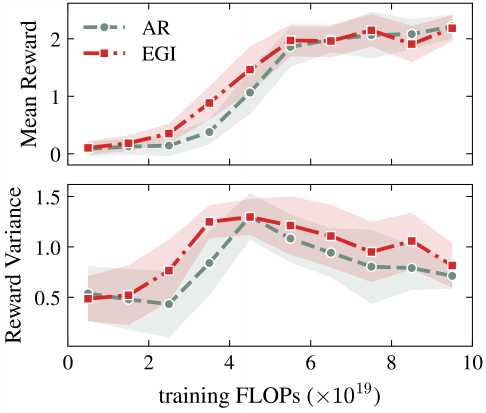}
  \vspace{-2em}
\caption{\textbf{Reward statistics during JustGRPO training on GSM8K.}
Mean reward (top) and variance (bottom) for AR and \egi{} rollouts.
}
    \label{fig:gsm8k_grpo}
    \vspace{-3em}
\end{wrapfigure}

\Cref{tab:rl_method_comparison} shows that \egi{} consistently outperforms AR rollouts across benchmarks and all three policy optimization methods under matched training compute.
This suggests that its broader rollout exploration translates into better downstream optimization outcomes.

\Cref{fig:gsm8k_grpo} provides a closer look at the training dynamics.
Compared with AR, \egi{} achieves faster gains in mean reward while maintaining higher within-group reward variance.
Thus, EGI improves rollout reward while preserving the reward variation used for group-relative learning.
Additional training dynamics across benchmarks in
App.~\ref{app:subsec-grpo_additional_results}.

Notably, \egi{} modifies only the first step before returning to standard TPP.
That this limited use of order flexibility improves both rollout exploration and downstream policy optimization suggests that arbitrary-order generation offers a promising design space for diverse rollouts.

\section{Conclusion}
\label{sec:conclusion}
We distinguish two decoding rules that have often been conflated under the same label of confidence-based decoding: low-confidence remasking (LCR) and top-probability position selection (TPP).
Through toy models and experiments with a large MDM, we show that LCR can severely suppress the lower-probability token exploration introduced by temperature, as cross-position rejection prevents most such proposals from reaching commitment and limits rollout diversity.
In contrast, TPP directly follows the tempered distribution at the selected position and achieves Pass@$k$ comparable to AR, while using order flexibility only at the first commitment to promote exploration further improves rollout diversity and downstream policy optimization.
These results highlight the importance of distinguishing LCR from TPP and suggest that arbitrary-order flexibility offers a promising design space for diverse rollout generation beyond AR.

\clearpage

\subsection*{AI Use Statement}
We used generative AI tools to assist with writing and editing the manuscript and with implementing and debugging experimental code.
All AI-assisted outputs were reviewed, revised, and cross-checked by the authors.
We take responsibility for the final content of this work.

\subsection*{Ethics Statement}
We are not aware of any specific ethical concerns raised by this work.

\subsection*{Reproducibility Statement}
Complete proofs of all theoretical results are provided in App.~\ref{app:theory}.
We also provide the experimental settings and implementation details needed to reproduce our results in Apps.~\ref{app:experimental-details} and~\ref{app:sec-exp-details-grpo}.

\bibliographystyle{plainnat}
\bibliography{iclr2027/iclr2027_conference}

\clearpage
\appendix

\section{Related Work}
\label{app:related-work}

\subsection{Prior Work on Diversity in Arbitrary-Order Decoding}
\label{app:related-flexibility}

Before discussing the closest prior work in detail, we summarize our positioning along three axes:

\begin{itemize}[leftmargin=*]
    \item \textbf{TPP, LCR, and their conflation.}
    The operational distinction between selecting among sampled tokens and selecting a position before token sampling is not itself new.
    \citet{hayakawa2026demystifying,fang2026locally} explicitly distinguish sample-then-choose from choose-then-sample procedures, while \citet{zhang2026differences} show that resampling a token after LCR selects a position largely restores $n$-gram entropy.
    None of these works, however, focuses on the existing TPP and LCR rules being conflated under \emph{confidence-based decoding}.
    Our focus is this conflation and why it matters for rollout sampling: the two rules define substantially different samplers when $T>0$, so conflating them can lead to misleading conclusions.

    \item \textbf{Diversity suppression under LCR.}
    Prior work reports weak Pass@$k$ and limited rollout coverage under LCR or closely related confidence-based decoding~\citep{ni2026flexibility,fang2026locally,olausson2026twoTemperatures}, providing the empirical starting point for our analysis.
    Most directly, \citet{zhang2026differences} theoretically establish entropy reduction under independent positions and empirically observe reduced $n$-gram entropy.
    We sharpen this analysis in an i.i.d.\ token model by showing exponential suppression of lower-probability choices and vanishing non-top fractions and normalized sequence entropy.
    We further show in an 8B LLaDA model that, even as sampling produces many lower-probability proposals, LCR commitments remain overwhelmingly concentrated on top tokens, and connect this suppression directly to reasoning rollout coverage.
    
    \item \textbf{Implications for arbitrary-order diversity.}
    Prior work responds to limited exploration in different ways, including favoring AR generation~\citep{ni2026flexibility}, introducing a separate position temperature~\citep{olausson2026twoTemperatures}, and approximately targeting a globally tempered distribution~\citep{fang2026locally}.
    Our results show that much of the lost coverage can instead be recovered simply with TPP, which avoids LCR's proposal rejection and directly follows the requested token distribution at the selected position.
    EGI then uses order flexibility at only the first commitment to further improve coverage, highlighting arbitrary-order generation as a promising design space for diverse rollouts.
\end{itemize}

\citet{ni2026flexibility} report lower Pass@$k$ under arbitrary-order (AO) than AR decoding, using LCR as their main AO rule and entropy-based, margin-based, and random-order decoding as additional baselines.
They explain this gap through \emph{entropy degradation}, where confidence-based AO decoding postpones uncertain forking tokens until surrounding context narrows their possible branches.
However, our results show that this conclusion is decoding-rule dependent.
TPP still prioritizes high-top-probability positions and can postpone uncertain forks, yet achieves Pass@$k$ close to AR, while EGI further improves rollout diversity and coverage over AR.
These results indicate that the AR--LCR gap is driven primarily by LCR's rejection of lower-probability proposals rather than confidence-prioritizing ordering itself, and that AO flexibility remains a promising design space for diverse rollout sampling.

\citet{zhang2026differences} provide the most directly related analysis of bias induced by LCR's proposal rejection.
They attribute the bias to proposal-dependent selection favoring higher-probability outcomes, and prove the resulting entropy reduction under independent positions.
They further show that resampling the token after LCR selects a position largely restores $n$-gram entropy, closely paralleling our observation that removing proposal rejection with TPP recovers rollout coverage.
We quantify this suppression more sharply: in our i.i.d.\ model, non-top commitments are exponentially suppressed as more positions compete, and their final fraction vanishes as $O(1/L)$ (Prop.~\ref{prop:app-token-counts}).
At the entropy level, whereas \citet{zhang2026differences} establish entropy reduction, we further show that the normalized sequence entropy itself vanishes as $O(\log L/L)$ (Prop.~\ref{prop:app-toy-i-entropy}).
We further observe in an 8B MDM that temperature produces many lower-probability proposals while LCR commitments remain overwhelmingly concentrated on top tokens, and connect this proposal rejection to rollout coverage and downstream policy optimization.

\citet{hayakawa2026demystifying} explicitly distinguish the sample-then-choose structure of MaskGIT from choose-then-sample procedures and show through their moment-sampler analysis that proposal-dependent position selection can implicitly sharpen the token distribution.
However, their analysis of the MaskGIT sampler with position temperature $\alpha$ does not cover the $\alpha\to0$ limit corresponding to LCR.
Our work connects this operational distinction to the conflation of TPP and LCR under confidence-based decoding, and directly compares the two rules at the same token temperature to quantify the resulting suppression of non-top tokens under LCR.

\citet{fang2026locally} analyze the quality--exploration trade-off using a TPP-style position-first setup and relate the resulting entropy bounds to the behavior of LCR (see App.~\ref{app:related-conflation}).
Our analysis instead models LCR's proposal rejection directly and identifies cross-position rejection, rather than confidence-prioritized ordering itself, as the dominant source of diversity suppression in our toy settings.
They improve Pass@$k$ with an Independent Metropolis--Hastings sampler that approximately targets a globally tempered distribution.
In contrast, we first remove LCR's proposal-rejection pathology with TPP and then obtain substantial additional gains from EGI, a one-step modification that highlights the untapped potential of AO flexibility.

\citet{olausson2026twoTemperatures} separate token and position temperature as controls over what to sample and where to decode, and use position temperature to improve rollout diversity.
Their token-temperature analysis, however, uses a TPP-style position-first score and shows that, under its assumptions, all anchors precede the fork for every $T>0$ (see App.~\ref{app:related-conflation}).
Under LCR, token temperature changes the sampled proposals and substantially diversifies commitment order even while committed tokens remain overwhelmingly top tokens, providing much of LCR's remaining rollout variation (App.~\ref{app:lcr-order-diversity}).
Thus, token and position temperature affect overlapping aspects of LCR, while neither generally eliminates the bias from proposal rejection at finite temperatures.
Our analysis identifies this rejection as the dominant source of diversity suppression in our toy settings and removes it with TPP.

\subsection{Conflation of TPP and LCR in Prior Work}
\label{app:related-conflation}

TPP and LCR have both been described using terms such as \textit{confidence-based decoding} or \textit{confidence-based unmasking}.
Position-first decoding as in TPP appears in \citet{wu2026fast,ben-hamu2025accelerated}, whereas proposal-based LCR is used under similar terminology in \citet{wang2026spg,zhao2025d1,tang2025wd1,zhan2026simple}.
Note that these works either evaluate only at $T=0$ or use implementations consistent with their descriptions, so the shared terminology is not problematic in practice.
In several works, however, the rule described in the paper differs from that used in a released $T>0$ implementation.
The implementation comparisons below refer to publicly released code paths and do not necessarily reconstruct the exact code used for every reported experiment.

\citet{wang2026revolutionizing} describe TPP-style position selection using the maximum token probability at each position, while their released Dream rollout path at $T=0.8$ uses LCR-style proposal rejection, ranking sampled proposals by their probabilities after temperature scaling and top-$p$/top-$k$ filtering.
\citet{yu2025dimple} describe TPP-style position selection based on proposal-independent, maximum-probability confidence in Alg.~1, while their released positive-temperature confidence path uses LCR-style proposal rejection based on the probabilities of sampled proposals.
\citet{wang2026d2} describe d2-AnyOrder using TPP-style position selection before token sampling, whereas the released $T=0.9$ rollout implementation uses LCR-style proposal rejection, committing only the two highest-scoring proposals.
\citet{kimklass} define TPP-style position confidence by the maximum token probability in Def.~4.1, while the released Dream Top-$k$ confidence baseline at $T=0.2$ uses LCR-style proposal rejection based on sampled-proposal confidence.
Under matched scoring and position-selection settings, these TPP-style and LCR-style decoding rules coincide at $T=0$ but generally induce different sampling distributions at $T>0$.

The same distinction is needed when interpreting theoretical analyses.
\citet{fang2026locally} derive an entropy bound under the condition $\max_{v\in\mathcal V}p_i(v)\geq1-\delta$, with the committed token at position $i$ sampled from $p_i$.
They distinguish sample-then-filter from rank-then-sample procedures, but the resulting entropy bound uses a TPP-style position-first abstraction that does not explicitly model LCR's proposal rejection.
Similarly, Prop.~2 of \citet{olausson2026twoTemperatures} shows that, under an additional logit-gap assumption and with one token unmasked per step, all anchors precede the fork for every $T>0$, but analyzes the position-first score $c_i=\max_v p_i^{(T)}(v)$ rather than LCR's sampled-proposal confidence.
Thus, both analyses use position-first abstractions to interpret behavior under LCR, where proposal sampling and cross-position rejection define a different stochastic decoding rule.
By directly analyzing the LCR procedure and comparing it with TPP, we identify cross-position proposal rejection as a substantial contributor to diversity suppression in our settings.

\subsection{Clarifying the Origins of TPP and LCR}
\label{app:related-origins}

We use \citet{kim2025train} and \citet{nie2025large} as convenient reference formulations of TPP and LCR, respectively, rather than as claims about their independent origins.
\citet{kim2025train} study adaptive token ordering for MDMs, and their Top-$K$ probability rule selects positions by maximum token probability before sampling tokens.
For one-position commitment, this is the rule we call TPP.
\citet{nie2025large} describe low-confidence remasking following MaskGIT~\citep{chang2022maskgit}.
Their released $T>0$ path first samples proposals and then selects positions using the probabilities of those proposals, which is the rule we call LCR.

Both procedures have broader precedents in iterative masked prediction and masked diffusion~\citep{ghazvininejad2019maskpredict,chang2022maskgit,zheng2024a}.
Our terminology is therefore not intended to assign either decoding rule to a single originating work.
It is intended to make explicit their operational difference at $T>0$, where TPP selects a position before token sampling while LCR samples proposals before deciding which one to commit.

\clearpage
\section{Proofs and Additional Results}
\label{app:theory}

Both toy models follow the setups in Sec.~4 and use the TPP/LCR rules in Sec.~\ref{sec:tempered-decoding-rules}.
We start from a fully masked sequence and commit one token per step at temperature $T\in(0,\infty]$.
LCR proposals are independently redrawn at each step conditional on the current sequence.

\subsection{Toy Model I: Independent Tokens}
\label{app:independent-tokens}

\paragraph{Setup.}
We use the setting of Sec.~4.1, where every position independently follows the same strictly positive distribution $p$ over a finite vocabulary $\mathcal V$ with a unique top-probability token.
Recall that
\[
c_T:=\max_{v\in\mathcal V}p^{(T)}(v),
\]
and let $N_{\text{non-top}}$ denote the number of non-top tokens in the final length-$L$ sequence.

\begin{proposition}[Non-top token probability]
\label{prop:app-token-probability}
With $m$ masked positions remaining,
\begin{equation}
\label{eq:app-token-probability}
\mathbb P_{\mathrm{TPP}}(\text{non-top commitment})
=
1-c_T,
\qquad
\mathbb P_{\mathrm{LCR}}(\text{non-top commitment})
=
(1-c_T)^m.
\end{equation}
\end{proposition}

\begin{proof}
Under TPP, the selected position is sampled directly from $p^{(T)}$, so a non-top token is committed with probability $1-c_T$.

Under LCR, any proposal of the unique top token rejects every non-top proposal.
A non-top token is therefore committed if and only if all $m$ positions propose non-top tokens.
Since the proposals are independent draws from $p^{(T)}$,
\[
\mathbb P_{\mathrm{LCR}}(\text{non-top commitment})
=
(1-c_T)^m.
\]
\end{proof}

\begin{proposition}[Expected non-top fraction]
\label{prop:app-token-counts}
The expected non-top fractions satisfy
\begin{equation}
\label{eq:app-token-counts}
\mathbb E_{\mathrm{TPP}}
\left[
\frac{N_{\text{non-top}}}{L}
\right]
=
1-c_T,
\qquad
\mathbb E_{\mathrm{LCR}}
\left[
\frac{N_{\text{non-top}}}{L}
\right]
=
\frac{(1-c_T)[1-(1-c_T)^L]}{Lc_T}
\leq
\frac{1-c_T}{Lc_T}
=
O\!\left(\frac1L\right).
\end{equation}
\end{proposition}

\begin{proof}
By Proposition~\ref{prop:app-token-probability}, the probability of a non-top commitment is $1-c_T$ under TPP and $(1-c_T)^m$ under LCR when $m$ positions remain.
Summing over the $L$ decoding steps gives
\[
\mathbb E_{\mathrm{TPP}}[N_{\text{non-top}}]
=
L(1-c_T),
\qquad
\mathbb E_{\mathrm{LCR}}[N_{\text{non-top}}]
=
\sum_{m=1}^{L}(1-c_T)^m
=
\frac{(1-c_T)[1-(1-c_T)^L]}{c_T}.
\]
Dividing by $L$ gives the result.
\end{proof}

\subsection{Toy Model II: Non-Overlapping Sequences}
\label{app:nonoverlapping-sequences}

\paragraph{Setup.}
We use the setting of Sec.~4.2, where the candidate length-$L$ sequences differ at every position and there is a unique top-probability sequence.
Recall that $\pi_T$ denotes the probability of this sequence after tempering.

\begin{proposition}[Non-top sequence probability]
\label{prop:app-sequence-probability}
The probability of generating a non-top sequence is
\begin{equation}
\label{eq:app-sequence-probability}
\mathbb P_{\mathrm{TPP}}(\text{non-top sequence})
=
1-\pi_T,
\qquad
\mathbb P_{\mathrm{LCR}}(\text{non-top sequence})
=
(1-\pi_T)^L.
\end{equation}
\end{proposition}

\begin{proof}
Because any two candidate sequences differ at every position, each token at any position identifies exactly one sequence.
The first committed token therefore determines the entire generated sequence.

Before the first commitment, every position has the same top-token probability.
Under TPP, any position may be selected and its token is sampled directly from the tempered distribution.
The first token therefore identifies a non-top sequence with probability $1-\pi_T$.

Under LCR, a proposal belonging to the top-probability sequence has higher model probability than any proposal belonging to a non-top sequence.
A non-top sequence is therefore selected if and only if all $L$ initial proposals belong to non-top sequences.
Since these proposals are independent,
\[
\mathbb P_{\mathrm{LCR}}(\text{non-top sequence})
=
(1-\pi_T)^L.
\]
\end{proof}

\subsection{Additional Results}
\label{app:additional-results}

The token-level suppression in Sec.~4.1 also implies a stronger distributional consequence.
Although TPP retains constant entropy per token, the entropy per token under LCR vanishes as the sequence length grows.

\begin{proposition}[Final-sequence entropy]
\label{prop:app-toy-i-entropy}
Let $X^{1:L}$ denote the final sequence in Toy Model I.
Then
\[
H_{\mathrm{TPP}}(X^{1:L})
=
L H(p^{(T)}),
\qquad
H_{\mathrm{LCR}}(X^{1:L})
=
O(\log L),
\]
so $H_{\mathrm{LCR}}(X^{1:L})/L=O(\log L/L)\to0$.
\end{proposition}

\begin{proof}
Under TPP, every position is sampled directly from $p^{(T)}$, so
\[
X^{1:L}\sim(p^{(T)})^{\otimes L},
\]
which gives
\[
H_{\mathrm{TPP}}(X^{1:L})
=
L H(p^{(T)}).
\]

Under LCR, let $r_i=\mathbb P(X^i\neq a)$.
For each position,
\[
H(X^i)
\leq
h(r_i)+r_i\log(|\mathcal V|-1),
\]
where $h$ is the binary entropy function.
By subadditivity of entropy and concavity of $h$,
\[
H_{\mathrm{LCR}}(X^{1:L})
\leq
Lh\!\left(
\frac{\mathbb E_{\mathrm{LCR}}[N_{\text{non-top}}]}{L}
\right)
+
\mathbb E_{\mathrm{LCR}}[N_{\text{non-top}}]
\log(|\mathcal V|-1).
\]
Proposition~\ref{prop:app-token-counts} gives
$\mathbb E_{\mathrm{LCR}}[N_{\text{non-top}}]=O(1)$,
so the right-hand side is $O(\log L)$.
\end{proof}

\paragraph{Expanded numerical examples.}
Tab.~\ref{tab:toy-examples-expanded} expands the examples in Tab.~\ref{tab:toy-examples} across sequence lengths and temperatures.
For reference, ``TPP-equiv.\ $T$'' is the TPP temperature that gives the same non-top probability as the corresponding LCR result.
For $K$ outcomes with top probability $w$ and the remaining probability divided equally among the other $K-1$ outcomes, it is computed as
\[
T_{\mathrm{TPP\text{-}equiv}}(u)=\frac{\log\frac{(K-1)w}{1-w}}{\log\frac{(K-1)(1-u)}{u}},
\]
where $u$ is the non-top probability to be matched.
All entries in Tab.~\ref{tab:toy-examples-expanded} are analytic evaluations rather than simulation estimates.

\begin{table}[p]
\centering
\caption{
\textbf{Expanded toy examples of diversity suppression under LCR.}
Panel (a) uses the same 20-token distribution as Toy Model I, with top probability $0.1$ and the remaining $0.9$ shared equally among 19 tokens.
Panel (b) uses the same ten non-overlapping sequences as Toy Model II, with top-sequence probability $0.2$ and the remaining $0.8$ shared equally among nine sequences.
Probabilities and expected fractions are reported in percent.
``TPP-equiv.\ $T$'' is the TPP temperature matching the corresponding LCR result.
}
\label{tab:toy-examples-expanded}

\small
\setlength{\tabcolsep}{3pt}
\renewcommand{\arraystretch}{1.12}
\vspace{5pt}

\begin{tabular*}{\textwidth}{@{\extracolsep{\fill}}ccrrrrrr@{}}
\multicolumn{8}{@{}l}{\textbf{(a) Independent Token Generation (Toy I)}}\\[4pt]
\toprule
& &
\multicolumn{3}{c}{$\mathbb P(\text{non-top token})$}
&
\multicolumn{3}{c}{$\mathbb E[N_{\text{non-top}}/L]$}
\\
\cmidrule(lr){3-5}
\cmidrule(l){6-8}

$L$
&
$T$
&
\multicolumn{1}{c}{LCR}
&
\multicolumn{1}{c}{Target (TPP)}
&
\multicolumn{1}{c}{TPP-equiv.\ $T$}
&
\multicolumn{1}{c}{LCR}
&
\multicolumn{1}{c}{Target (TPP)}
&
\multicolumn{1}{c}{TPP-equiv.\ $T$}
\\
\midrule

$32$ & $0.5$ & $0.1179$ & $81.00$ & $0.0771$ & $13.31$ & $81.00$ & $0.1551$ \\
     & $1$ & $3.434$ & $90.00$ & $0.1190$ & $27.16$ & $90.00$ & $0.1901$ \\
     & $2$ & $9.46$ & $92.90$ & $0.1436$ & $37.00$ & $92.90$ & $0.2149$ \\
     & $\infty$ & $19.37$ & $95.00$ & $0.1710$ & $47.87$ & $95.00$ & $0.2466$ \\

\addlinespace[4pt]

$128$ & $0.5$ & $1.932\times10^{-10}$ & $81.00$ & $0.0250$ & $3.33$ & $81.00$ & $0.1184$ \\
      & $1$ & $1.390\times10^{-4}$ & $90.00$ & $0.0455$ & $7.03$ & $90.00$ & $0.1352$ \\
      & $2$ & $8.010\times10^{-3}$ & $92.90$ & $0.0604$ & $10.22$ & $92.90$ & $0.1460$ \\
      & $\infty$ & $0.1408$ & $95.00$ & $0.0786$ & $14.82$ & $95.00$ & $0.1592$ \\

\addlinespace[4pt]

$512$ & $0.5$ & $1.394\times10^{-45}$ & $81.00$ & $0.0067$ & $0.83$ & $81.00$ & $0.0967$ \\
      & $1$ & $3.734\times10^{-22}$ & $90.00$ & $0.0131$ & $1.76$ & $90.00$ & $0.1072$ \\
      & $2$ & $4.117\times10^{-15}$ & $92.90$ & $0.0184$ & $2.55$ & $92.90$ & $0.1135$ \\
      & $\infty$ & $3.931\times10^{-10}$ & $95.00$ & $0.0256$ & $3.71$ & $95.00$ & $0.1205$ \\

\bottomrule
\end{tabular*}

\vspace{14pt}

\begin{tabular*}{\textwidth}{@{\extracolsep{\fill}}ccrrr@{}}
\multicolumn{5}{@{}l}{\textbf{(b) Non-Overlapping Sequence Generation (Toy II)}}\\[4pt]
\toprule
& &
\multicolumn{3}{c}{$\mathbb P(\text{non-top sequence})$}
\\
\cmidrule(l){3-5}

$L$
&
$T$
&
\multicolumn{1}{c}{LCR}
&
\multicolumn{1}{c}{Target (TPP)}
&
\multicolumn{1}{c}{TPP-equiv.\ $T$}
\\
\midrule

$32$ & $0.5$ & $6.277\times10^{-5}$ & $64.00$ & $0.0492$ \\
     & $1$ & $0.07923$ & $80.00$ & $0.0869$ \\
     & $2$ & $0.7206$ & $85.71$ & $0.1138$ \\
     & $\infty$ & $3.434$ & $90.00$ & $0.1465$ \\

\addlinespace[4pt]

$128$ & $0.5$ & $1.553\times10^{-23}$ & $64.00$ & $0.0137$ \\
      & $1$ & $3.940\times10^{-11}$ & $80.00$ & $0.0264$ \\
      & $2$ & $2.697\times10^{-7}$ & $85.71$ & $0.0370$ \\
      & $\infty$ & $1.390\times10^{-4}$ & $90.00$ & $0.0517$ \\

\addlinespace[4pt]

$512$ & $0.5$ & $5.810\times10^{-98}$ & $64.00$ & $0.0035$ \\
      & $1$ & $2.410\times10^{-48}$ & $80.00$ & $0.0070$ \\
      & $2$ & $5.287\times10^{-33}$ & $85.71$ & $0.0100$ \\
      & $\infty$ & $3.734\times10^{-22}$ & $90.00$ & $0.0144$ \\

\bottomrule
\end{tabular*}

\end{table}

\clearpage

\section{Experimental Details}
\label{app:experimental-details}

\subsection{Proposal and Commitment Diagnostics}
\label{app:lcr-details}

\begin{figure*}[t]
    \centering
    \includegraphics[width=\textwidth]{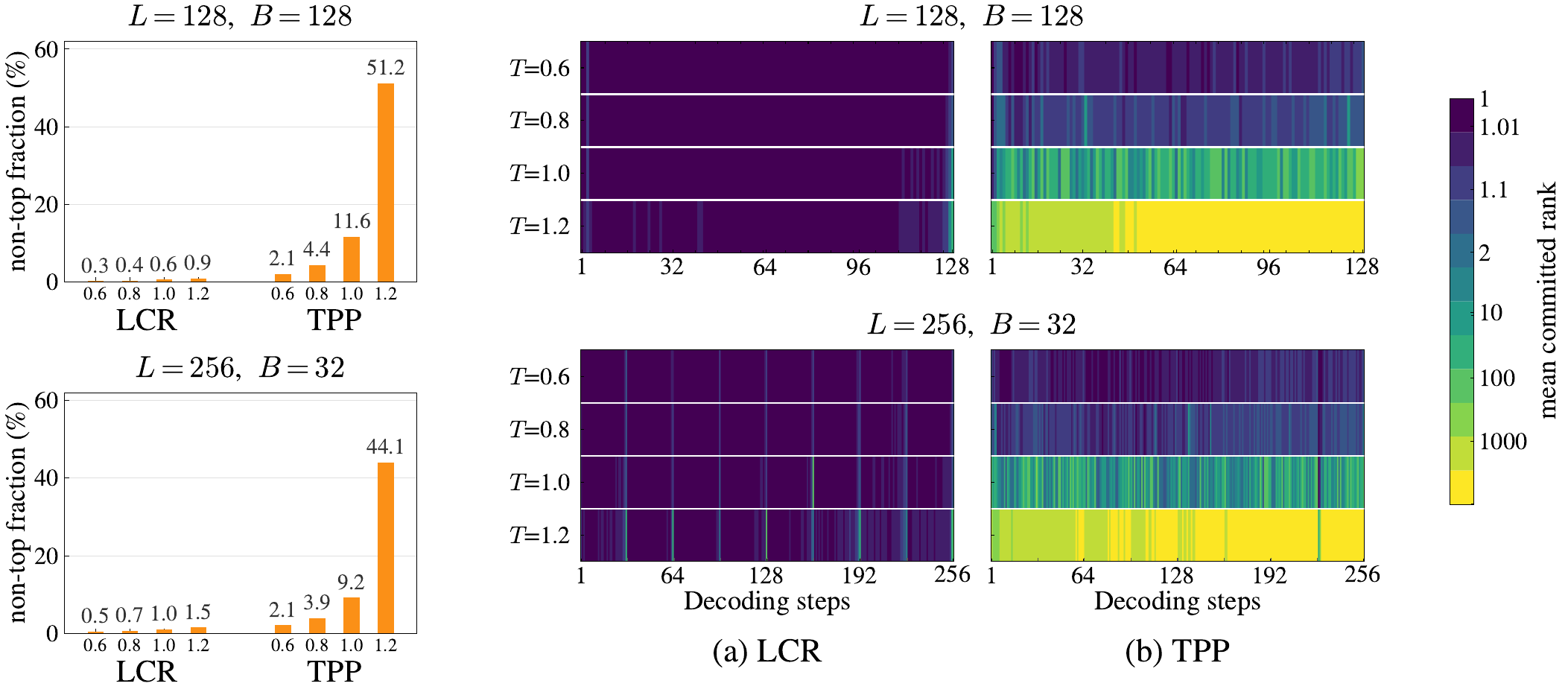}
    
\caption{
(\textbf{Left}) \textbf{Non-top commitment rate.}
Higher temperature substantially increases non-top commitments under TPP but barely under LCR.
Bars show the fraction of committed tokens that are non-top.
(\textbf{Right}) \textbf{Committed-token rank over decoding.}
LCR remains near rank one until the end of each sequence or block, whereas TPP commits higher-rank tokens throughout decoding as temperature increases.
Heatmaps show the per-step mean local rank of committed tokens.
Both panels use HumanEval with $L/B=128/128$ and $256/32$.}

\vspace{0em}
    \label{fig:non_top_fraction_he}
\end{figure*}

\begin{wrapfigure}{r}{0.4\columnwidth}
    \centering
    \vspace{-1.5em}
    \includegraphics[width=\linewidth]{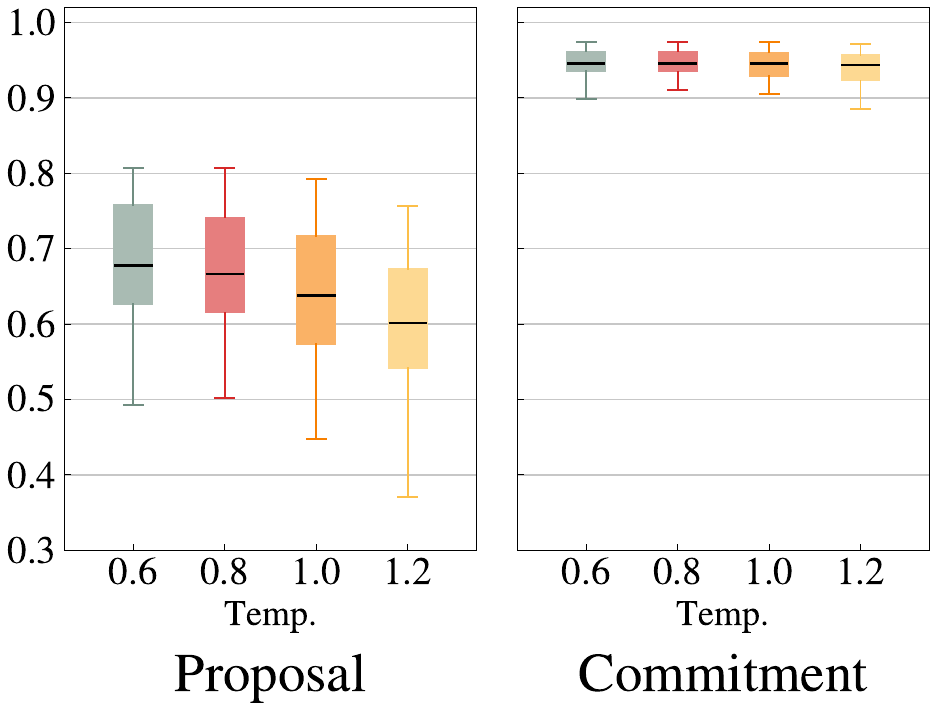}
    \vspace{-2em}
\caption{
\textbf{LCR\hspace{-0.1em} proposal/commitment.}
Box plots show untempered model probabilities of proposed and committed tokens,
pooled across decoding steps on HumanEval ($L/B=256/32$).
}
\label{fig:lcr-proposal-commit-he}
    \vspace{-2em}
\end{wrapfigure}

We evaluate LLaDA-8B-Instruct on GSM8K and HumanEval using $64$ rollouts
per prompt, with generation length/block size configurations of
$128/128$ and $256/32$.
Across both benchmarks and configurations, LCR commits almost exclusively
to the local argmax, the highest-probability token at each position
at the time of commitment.
Non-top commitments become appreciable only when very few masked positions remain, typically two or three
(Figs.~\ref{fig:non_top_fraction_gsm}
and~\ref{fig:non_top_fraction_he}).
Departures from the local argmax are thus concentrated near sequence
or block boundaries, suggesting that cross-position competition
suppresses them while many positions remain available.

The gap between proposals and commitments reveals how this behavior arises.
LCR frequently proposes tokens with untempered model probabilities
around $0.6$--$0.7$, yet most lose to higher-probability proposals
at competing positions.
Regardless of the sampling temperature tested, committed tokens concentrate at much higher probabilities, often above $0.9$
(Figs.~\ref{fig:lcr-proposal-commit-gsm}
and~\ref{fig:lcr-proposal-commit-he}).
These observations suggest that LCR's near-argmax behavior arises from selection at the commitment stage: variability in sampled
proposals need not translate into comparable variability in
committed tokens.

\subsection{Why LCR Still Gains from Multiple Rollouts}
\label{app:lcr-order-diversity}

Although LCR rejects most lower-probability token proposals, its rollouts
are not identical.
Tab.~\ref{tab:lcr-order-diversity} shows that increasing temperature
increases both the normalized edit distance between generation orders
and the token-level Hamming distance between final outputs.
Here, normalized edit distance counts insertions and deletions divided
by the combined sequence length, while normalized Hamming distance
measures the fraction of token positions that differ, 
where higher distance indicates greater diversity.

The two distances are strongly correlated across rollout pairs
($0.79$--$0.96$).
Different generation orders change the contexts available for later
predictions, allowing order variation to propagate into different outputs.

Meanwhile, nearly all committed tokens remain the local argmax
across temperatures
(Figs.~\ref{fig:non_top_fraction_gsm}
and~\ref{fig:non_top_fraction_he}).
Together, these observations suggest that generation order is a major
source of LCR's remaining rollout variation, helping explain why
Pass@$k$ can improve with additional rollouts despite nearly greedy
token commitments.
However, this variation primarily explores different near-greedy
trajectories, leaving alternatives that require lower-probability
token commitments less explored.

\begin{table}[t]
\centering
\caption{
\textbf{Generation-order and output diversity under LCR.}
For $64$ rollouts per prompt, \emph{Norm. Edit} and \emph{Norm. Hamming} report the mean pairwise normalized edit distance between generation orders and the normalized Hamming distance between final token sequences, respectively, with values shown as mean $\pm$ SE. \emph{Corr.} denotes the Pearson correlation between the two distances across pairs.
}
\label{tab:lcr-order-diversity}
\vspace{0.4em}

\footnotesize
\setlength{\tabcolsep}{4pt}
\renewcommand{\arraystretch}{1.08}

\begin{tabular*}{\columnwidth}{
@{\extracolsep{\fill}}
lcccc
@{}
}
\toprule
\textbf{Dataset}
& \textbf{$T$}
& \textbf{Norm. Edit}
& \textbf{Norm. Hamming}
& \textbf{Corr.}
\\
\midrule

\multirow{4}{*}{GSM8K}
& $0.6$ & $0.32 \pm 0.004$ & $0.51 \pm 0.007$ & $0.94$ \\
& $0.8$ & $0.37 \pm 0.003$ & $0.59 \pm 0.006$ & $0.92$ \\
& $1.0$ & $0.42 \pm 0.003$ & $0.66 \pm 0.006$ & $0.86$ \\
& $1.2$ & $0.45 \pm 0.002$ & $0.71 \pm 0.005$ & $0.79$ \\

\midrule

\multirow{4}{*}{HumanEval}
& $0.6$ & $0.33 \pm 0.009$ & $0.56 \pm 0.017$ & $0.96$ \\
& $0.8$ & $0.38 \pm 0.007$ & $0.64 \pm 0.014$ & $0.94$ \\
& $1.0$ & $0.42 \pm 0.006$ & $0.72 \pm 0.012$ & $0.92$ \\
& $1.2$ & $0.45 \pm 0.005$ & $0.76 \pm 0.010$ & $0.91$ \\

\bottomrule
\end{tabular*}

\end{table}

\subsection{Solution Coverage Comparisons}
\label{app:coverage-details}

Following~\citet{ni2026flexibility}, we compare AR, LCR, TPP, and EGI at a baseline temperature of $T=0.6$ in Fig.~\ref{fig:passk-results}. For EGI, only the first commitment uses $T=0.9$; all subsequent commitments use $T=0.6$. 
We generate $n=64$ rollouts per prompt and report Pass@$k$ following~\citet{chen2021humaneval}: \[ \operatorname{Pass@}k = \mathbb{E}_{q} \left[ 1- \frac{\binom{n-c_q}{k}} {\binom{n}{k}} \right], \] where $c_q$ is the number of correct solutions among the $n$ samples for prompt $q$.

\subsection{Diversity and Potential Metrics}
\label{app:diversity_exp_detail}

We report three metrics:
Div-Self-BLEU measures surface-level diversity among generated solutions,
Potential@$k$ measures multi-rollout success with greater weight on
problems with low single-rollout success, and normalized edit distance
measures diversity in commitment orders.

\paragraph{Div-Self-BLEU.}
Following~\citet{yao2025diversity}, we report Div-Self-BLEU to measure
inter-rollout diversity.
For each prompt, each generated response is treated as a hypothesis while
each of the remaining responses serves as a single reference in turn, and the
resulting BLEU scores are averaged across response pairs.
Since a lower Self-BLEU indicates greater diversity, we report
\[
    \mathrm{Div\text{-}Self\text{-}BLEU}
    =
    100 - \mathrm{Self\text{-}BLEU}.
\]
Div-Self-BLEU therefore ranges from $0$ to $100$, with higher values indicating
greater diversity among the generated rollouts.
Refer to~\citet{wiher2022bleu} for additional details about BLEU metric.

\paragraph{Potential@$k$.}
We further report Potential@$k$~\citep{yao2025diversity}, which measures the
ability of additional rollouts to recover problems that are not solved by a
single rollout.
For $N$ evaluation problems, it is defined as
\[
    \mathrm{Potential@}K
    =
    \frac{
        \sum_{i=1}^{N}
        \mathrm{Pass@}K(q_i)
        \left(1-\mathrm{Pass@}1(q_i)\right)
    }{
        \sum_{i=1}^{N}
        \left(1-\mathrm{Pass@}1(q_i)\right)
    },
\]
where $q_i$ denotes the $i$-th problem.
Thus, Potential@$k$ focuses on the additional solution coverage obtained from
multiple rollouts beyond single-rollout performance.

Following~\citet{yao2025diversity}, we compute Pass@$1$ using deterministic LCR/TPP with $T=0$, under which TPP and LCR produce the same decoding behavior.

\paragraph{Normalized edit distance.}
We measure decoding-order diversity by comparing the sequences of
positions committed in different rollouts.
For two commitment orders $\pi$ and $\sigma$, we define
\[
    \operatorname{NED}(\pi,\sigma)
    =
    \frac{d_{\mathrm{ID}}(\pi,\sigma)}
         {|\pi|+|\sigma|},
\]
where $d_{\mathrm{ID}}$ is the minimum number of insertions and
deletions required to transform one sequence into the other;
substitutions are not allowed.
When both orders are permutations of the same $n$ positions,
this simplifies to
\[
    \operatorname{NED}(\pi,\sigma)
    =
    1-\frac{\operatorname{LCS}(\pi,\sigma)}{n},
\]
where $\operatorname{LCS}$ denotes the length of the longest common
subsequence.
For each prompt, we average NED over all unordered pairs of rollouts,
then average the resulting values across prompts.
A value of zero indicates identical commitment orders, while larger
values indicate greater order diversity.

For example, consider commitment orders $\pi=(1,2,3,4)$ and
$\sigma=(2,3,4,1)$.
Deleting position $1$ from the beginning and inserting it at the end
requires two edits, giving $\operatorname{NED}=2/(4+4)=0.25$.
Equivalently, their longest common subsequence is $(2,3,4)$,
so $\operatorname{NED}=1-3/4=0.25$.

\subsection{Effect of First-Step Temperature in AR}
\label{app:ar-firsttoken}

To test whether EGI's gains arise simply from increasing the first-step sampling temperature, we evaluate an AR variant that raises the temperature to $0.9$ only at the first decoding step, matching \egi{}. All subsequent steps retain the standard AR temperature.

As shown in \Cref{fig:ar-fork-egi}, this modification has only a limited effect, with Pass@$k$ remaining close to that of standard AR. 
One possible explanation is that AR always begins at the fixed position immediately following the prompt, where predictive uncertainty may already be relatively low. 
Increasing the temperature at this position may therefore induce only limited additional variation. By comparison, \egi{} uses order flexibility to select a high-entropy initial position, where sampling at a higher temperature may have a larger effect on subsequent trajectories. This comparison suggests that the choice of the initial position may play an important role beyond the temperature increase alone.

\begin{figure*}[t]
    \centering
    \includegraphics[width=\textwidth]{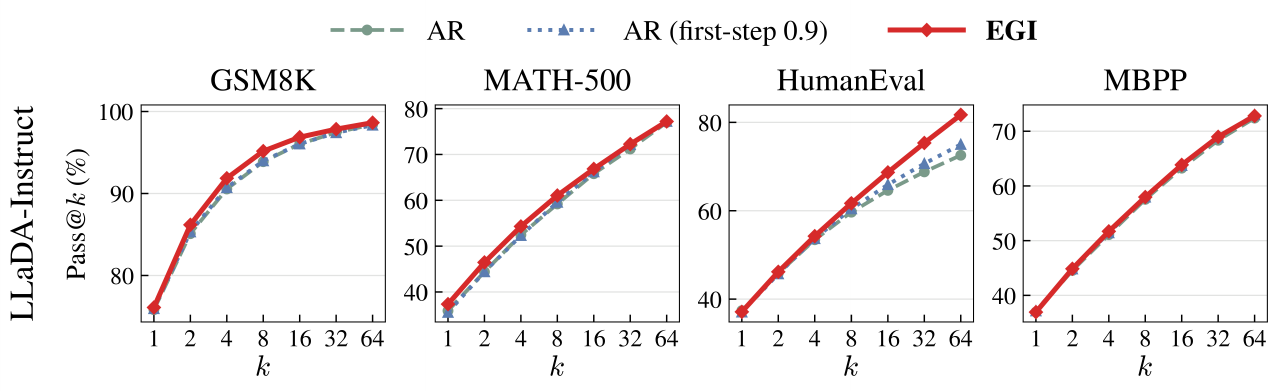}
\vspace{-1em}
\caption{\textbf{Effect of first-step temperature on Pass@$\bm{k}$.}
Raising AR's first-step temperature to $0.9$ yields performance close to standard AR, while \egi{} achieves higher Pass@$k$.
}
\vspace{-0.8em}
    \label{fig:ar-fork-egi}
\end{figure*}

\clearpage
\section{Experimental Details on GRPO}
\label{app:sec-exp-details-grpo}

\paragraph{Training configuration.}
Following \citet{zhao2025d1}, we equip the attention and MLP projections
of \lladainst{} with LoRA adapters~\citep{hu2022lora}, using rank $128$,
$\alpha = 64$, and dropout $0.05$, while keeping the backbone frozen
in 4-bit precision.
This parameter-efficient setup has also been adopted or extended by
\citet{tang2025wd1,wang2026spg,wang2026d2}.
For consistency, we adapt the RL training pipeline of
\citet{ni2026flexibility} from full-parameter fine-tuning to this LoRA setup,
preserving its likelihood definition and objective form.

Across all three GRPO objectives, we use a learning rate of $3 \times 10^{-6}$
and four inner iterations per rollout batch.
Each optimizer update uses eight prompts with eight rollouts per prompt,
giving an effective batch size of $64$ completions.
Prompt tokens are randomly masked with probability $0.15$.

The applicability of clipping and KL regularization depends on the
objective.
Following \citet{zhao2025d1}, we use a clipping range of $\epsilon = 0.5$
for objectives with importance-ratio clipping and a KL coefficient of
$\beta = 0.04$ for objectives with a KL term.

\paragraph{Rollout configuration.}
We use a maximum generation length of $256$, a block size of $32$,
and $8$ rollouts per prompt in all training runs.
Following the standard \lladainst{} generation protocol~\citep{nie2025large},
we commit one masked token per decoding step, requiring $256$ steps
for a length-$256$ completion.
AR and EGI share these settings and differ only in the sampling rule,
requiring the same number of model evaluations and comparable
computational costs per rollout.

Following the best-performing AR configuration in terms of Pass@$k$
reported by \citet{ni2026flexibility}, we use a base temperature of
$T=0.6$ for both AR and EGI.
For EGI's initial token sampling, we use $T=0.9$ on math benchmarks
and $T=0.6$ on coding benchmarks.

\newenvironment{objlist}{%
  \begin{list}{\textbullet}{%
    \setlength{\leftmargin}{0.9em}%
    \setlength{\labelwidth}{0.45em}%
    \setlength{\labelsep}{0.35em}%
    \setlength{\itemindent}{0pt}%
    \setlength{\listparindent}{0pt}%
    \setlength{\rightmargin}{0pt}%
    \setlength{\topsep}{0.3\baselineskip}%
    \setlength{\partopsep}{0pt}%
    \setlength{\itemsep}{0.35\baselineskip}%
    \setlength{\parsep}{0.15\baselineskip}%
  }%
}{\end{list}}

\paragraph{Baselines.}
We compare three policy-optimization objectives, each crossed with both sampling rules.
All three are group-relative: for a query $q$, a group of $G$ responses is drawn from the old policy, and each response $o_i$ is assigned an advantage $\hat{A}$ by standardizing its reward against the group statistics.
They differ in how they handle the sequence likelihood.
Two of them share the clipped surrogate
\begin{equation}
\mathcal{C}(\rho, \hat{A}) = \min\big(\rho\hat{A},\;
\mathrm{clip}(\rho, 1-\epsilon, 1+\epsilon)\,\hat{A}\big),
\end{equation}
which we write once here to keep the objectives below compact.

\begin{objlist}
\item \textbf{SPG-MIX.}
\citet{wang2026spg} optimize bounds on the likelihood rather than an estimate of it,
choosing the bound by the sign of the advantage:
\begin{equation}
\mathcal{J}_{\mathrm{SPG}}(\theta)=\mathbb{E}\Big[\tfrac{1}{G}\textstyle\sum_{j}
\hat{A}^j \mathcal{L}^j\Big],
\end{equation}
where $\mathcal{L}^j=\mathcal{L}_{\mathrm{ELBO}}$ when $\hat{A}^j \ge 0$ and
$\tilde{\mathcal{L}}_{\mathrm{EUBO}}$ otherwise,
so a positively-advantaged trace raises a lower bound on its likelihood while a
negatively-advantaged one lowers an upper bound --- the sandwich the method is named for.
Writing $w(t)$ for the weighting at diffusion time $t$, $z_t$ for the noised sequence and
$m_{t,i}$ for the indicator that position $i$ is masked in $z_t$, the two bounds are
\begin{equation}
\mathcal{L}_{\mathrm{ELBO}}=\mathbb{E}_{t,z_t}\Big[\textstyle\sum_i w(t)\,m_{t,i}
\log \pi_\theta(x_i \mid z_t)\Big],
\end{equation}
\begin{equation}
\tilde{\mathcal{L}}_{\mathrm{EUBO}}=\tfrac{1}{\beta}\textstyle\sum_i \log
\mathbb{E}_{t,z_t}\Big[w(t)\,m_{t,i}\,\pi^{\beta}_\theta(x_i \mid z_t)\Big],
\end{equation}
with $\beta$ the bound temperature.
Their mixed variant, which we use and refer to as SPG-MIX, replaces the negative branch with
$\omega\tilde{\mathcal{L}}_{\mathrm{EUBO}}+(1-\omega)\mathcal{L}_{\mathrm{ELBO}}$.
We set $\omega = 0.5$ and the bound temperature to $1.5$, and estimate the likelihood at two
diffusion times drawn from $[0,1]$ under a block-random forward process.

\item \textbf{JustGRPO.}
\citet{ni2026flexibility} forgo arbitrary order during RL training, sampling rollouts
autoregressively so that the likelihood factorizes as
\begin{equation}
\log \pi^{\mathrm{AR}}_\theta(o\mid q)=\textstyle\sum_k
\log \pi^{\mathrm{AR}}_\theta(o_k\mid o_{<k},q),
\end{equation}
so that every generated position is scored exactly.
This admits standard GRPO~\citep{shao2024deepseekmath}:
\begin{equation}
\mathcal{J}(\theta)=\mathbb{E}\Big[\tfrac{1}{G}\textstyle\sum_{i,k}
\tfrac{1}{|o_i|} \mathcal{C}(\rho_{i,k}, \hat{A}_{i,k})\Big] - \beta D_{\mathrm{KL}},
\end{equation}
where $\rho_{i,k}$ is the token-level ratio of $\pi^{\mathrm{AR}}_\theta$ to
$\pi^{\mathrm{AR}}_{\theta_{\text{old}}}$ at $o_{i,k}$ given $o_{i,<k}$ and $q$.

\smallskip
\noindent
In our \egi{} condition, we change only the rollout sampler while retaining left-to-right likelihood evaluation and the same surrogate objective.
This introduces a mismatch between the rollout distribution and the
autoregressive likelihood used for optimization.
Accordingly, we treat this setting as a JustGRPO-style surrogate with \egi{} rollouts, allowing us to study the effect of the rollout sampler while keeping the optimization rule fixed.

\item \textbf{d2-StepMerge.}
\citet{wang2026d2} decompose the trajectory likelihood over $N$ merged segments rather than all $T$ decoding steps, replacing $\prod_{t}\pi_\theta(x_t \mid x_{t+1})$ with
$\prod_{n}\pi_\theta(x_{nT/N} \mid x_{(n+1)T/N})$, and optimize
\begin{equation}
\mathbb{E}\Big[\tfrac{1}{L}\textstyle\sum_{n,l}
\mathbf{1}_{n,l}\,\mathcal{C}(\rho^l_n, \hat{A}^l)\Big] - \beta D_{\mathrm{KL}},
\end{equation}
where $\rho^l_n$ is the ratio of $\pi_\theta$ to $\pi_{\text{old}}$ at
$x^l_{nT/N}$ given $x^{1:L}_{(n+1)T/N}$, and $\mathbf{1}_{n,l}$ selects the positions
unmasked in segment $n$.
Larger $N$ tightens the decomposition at the cost of more model passes.
Following the per-benchmark choice of \citet{wang2026d2} we use $N = 8$ for GSM8K and $N = 16$ for MATH-500 and code.
\end{objlist}

\paragraph{Benchmarks and rewards.}
We follow \citet{zhao2025d1} for math prompts and rewards, and
\citet{ni2026flexibility} for coding, adjusting only the coding format
reward to accept trailing text after the first code block, consistent
with the evaluation harness.
Prompts and rewards are identical across objectives within each benchmark.

\paragraph{Evaluation protocol.}
We compare all objectives and sampling rules under a matched training budget
of $1 \times 10^{20}$ accounted FLOPs, with evaluation every
$1 \times 10^{19}$ FLOPs.
Each run therefore has ten evaluation checkpoints, and we report the
highest accuracy among them.
Evaluation uses each benchmark's full test set with the same deterministic
decoding procedure across runs: LCR/TPP with temperature set to zero,
which are equivalent under greedy decoding, consistent with prior
post-RL evaluation protocols~\citep{zhao2025d1,tang2025wd1,ni2026flexibility}.
Training stops when the accumulated FLOP count reaches the budget;
the number of optimizer steps is thus determined by each method's
accounted compute cost rather than fixed in advance.

\paragraph{FLOP accounting.}
Following \citet{wang2026d2}, we use a model-based FLOP accounting convention.
We charge $2P$ per token for a no-gradient forward pass and $4P$ for an
evaluation including activation backpropagation through the frozen backbone,
where $P = 8 \times 10^{9}$.
The additional gradient cost of the comparatively small LoRA adapters is
neglected.
These charges provide a hardware-independent compute proxy, not an exact
measurement of implementation-level FLOPs.

The counter uses forward-row counts recorded by the implementation,
multiplied by the full sequence length $L$ (prompt plus completion) and
the applicable per-token charge.
Let $G$ denote the group size across data-parallel processes, $T$ the
decoding steps per rollout, and $M$ the per-completion likelihood
evaluations defined above.
The counted components are:
\begin{itemize}
    \item \textbf{Rollout generation:} $GT$ rows per generation batch,
    charged at $2P$ per token.
    \item \textbf{Old- and reference-policy likelihoods:} $GM$ rows each
    time either policy's likelihood estimator is evaluated,
    charged at $2P$ per token.
    \item \textbf{Current-policy likelihood:} $GM$ rows each time the
    likelihood estimator is evaluated with backpropagation,
    charged at $4P$ per token.
\end{itemize}
Costs accumulate whenever the corresponding computation occurs,
following each objective's update schedule.
Separate counters for rollout, old-policy, reference-policy, and
current-policy costs are aggregated across data-parallel processes and
logged at each training step, so the per-component breakdown is recorded for every run.

The budget excludes reward computation (including sandboxed code execution),
optimizer arithmetic, tokenization, and periodic evaluation.
It therefore measures accounted model computation during training rather
than total training cost.
The rollout protocol fixes $T = 256$.
Likelihood costs depend on both the update schedule and $M$:
$M = 2$ for SPG-MIX, $M = N$ for d2-StepMerge, and $M = 256$ for JustGRPO.

\paragraph{Hardware.}
Runs execute on a single node of four NVIDIA B200 GPUs under ZeRO-2 sharding with bfloat16
mixed precision.
The number of prompts behind an optimizer update is fixed rather than derived from the device
count, so the effective batch is identical across runs.

\subsection{Additional Training Dynamics}
\label{app:subsec-grpo_additional_results}

\begin{wrapfigure}{r}{0.55\columnwidth}
    \centering
    \vspace{-1.2em}
    \includegraphics[width=\linewidth]{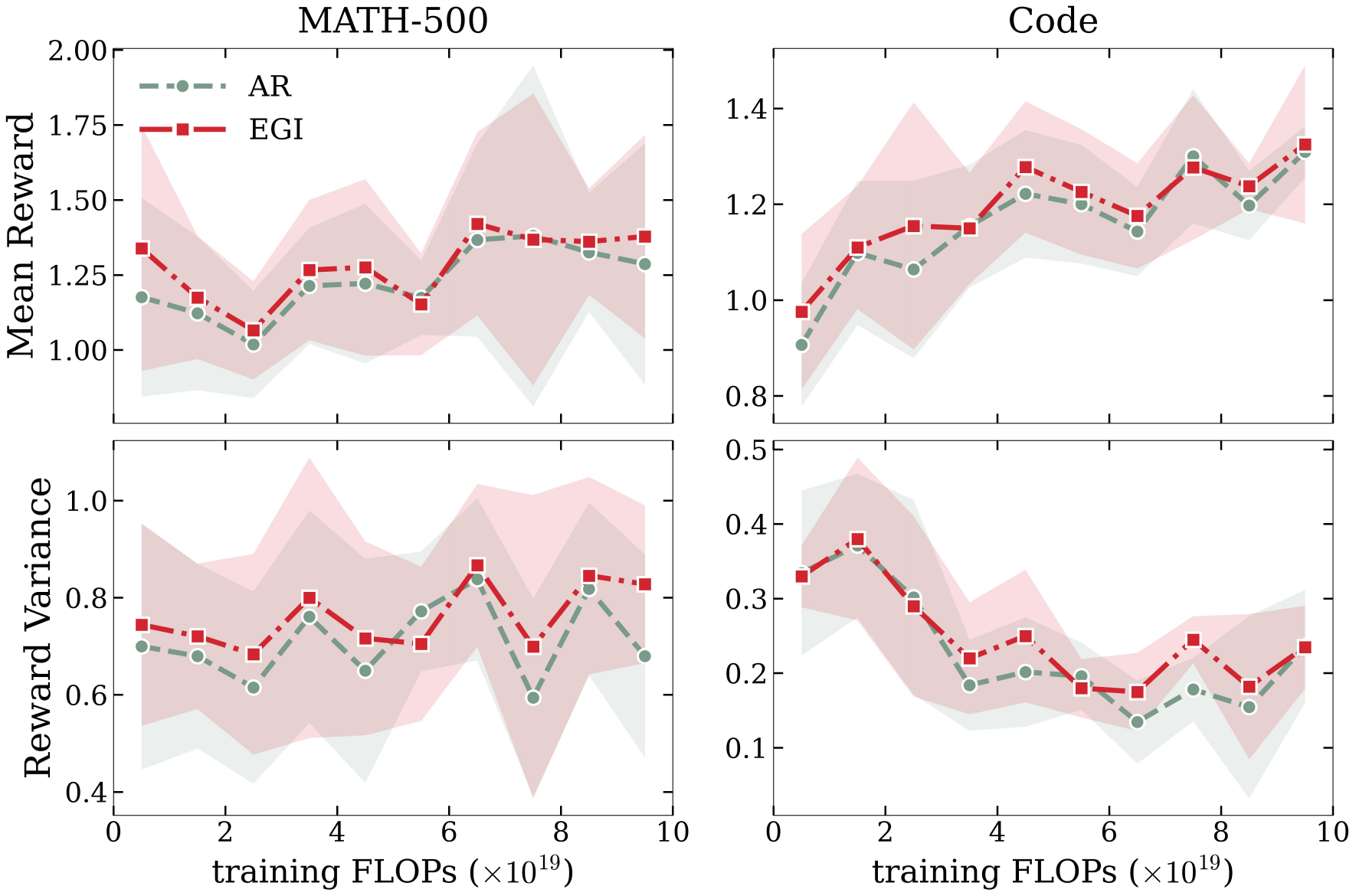}
  \vspace{-2em}
    \caption{
\textbf{Reward statistics during JustGRPO training.}
Mean rollout reward and within-group reward variance for AR and \egi{} on MATH-500 and an AceCoder-87K subset, plotted against cumulative training FLOPs.
    }
    \label{fig:grpo-dynamics-figs}
\end{wrapfigure}

We provide additional training dynamics beyond GSM8K on MATH-500 and coding tasks using a subset of AceCoder-87K.
\Cref{fig:grpo-dynamics-figs} compares the mean rollout reward and reward variance of AR and \egi{} under JustGRPO, plotted against accumulated training FLOPs.

Mean rollout reward tracks improvements in rollout quality, while within-group reward variance captures the reward differences that drive group-relative policy updates.
Across both settings, \egi{} tends to achieve higher mean rollout rewards than AR at comparable training FLOPs while maintaining within-group reward variation.
These results show a similar trend to that observed on GSM8K across additional mathematical reasoning and code generation task.

\end{document}